\documentclass[11pt]{article}

\usepackage{fullpage}

\usepackage[utf8]{inputenc} 
\usepackage[T1]{fontenc}    
\usepackage{times}
\usepackage{enumitem}
\usepackage[colorlinks = true]{hyperref}       
\usepackage{float}
\usepackage[numbers]{natbib}

\usepackage{datetime}
\usepackage{bbm}
\usepackage[normalem]{ulem}
\usepackage{authblk}
\usepackage{url}            
\usepackage{booktabs}       
\usepackage{amsthm}
\usepackage{wrapfig, blindtext}
\usepackage{amsfonts}       
\usepackage{bbold}
\usepackage{nicefrac}       
\usepackage{xcolor}
\usepackage{hyperref}
\usepackage{amsmath}
\usepackage{amssymb}
\usepackage{graphicx}
\usepackage{algorithm}
\usepackage{algorithmic}
\usepackage{wrapfig}
\usepackage{microtype}      
\usepackage{comment}

\theoremstyle{plain}
\newtheorem{theorem}{Theorem}[section]
\newtheorem{proposition}[theorem]{Proposition}
\newtheorem{lemma}[theorem]{Lemma}
\newtheorem{corollary}[theorem]{Corollary}
\theoremstyle{definition}
\newtheorem{definition}[theorem]{Definition}

\theoremstyle{remark}
\newtheorem{remark}[theorem]{Remark}

\definecolor{myblue}{RGB}{80,80,160}
\definecolor{mygreen}{RGB}{80,160,80}
\definecolor{light-gray}{gray}{0.4}

\newcommand{\nr}[1]{\textbf{\textcolor{magenta}{[nr:#1]}}}

\newcommand{\com}[1]{\small\color{light-gray}{#1}\color{black}\normalsize\\}
\newcommand{\algoname}{GRUB}

\title{Maximizing and Satisficing in Multi-armed Bandits with Graph Information}

\author[1]{Parth~K.~Thaker}
\author[1]{Mohit~Malu}
\author[2]{Nikhil~Rao}
\author[1]{Gautam~Dasarathy}

\affil[1]{School of Electrical, Computer and Energy Engineering, Arizona State University}
\affil[2]{Microsoft}

\begin{document}

\maketitle

\begin{abstract}
Pure exploration in multi-armed bandits has emerged as an important framework for modeling decision making and search under uncertainty. In modern applications however, one is often faced with a tremendously large number of options and even obtaining one observation per option may be too costly rendering traditional pure exploration algorithms ineffective. Fortunately, one often has access to similarity relationships amongst the options that can be leveraged. In this paper, we consider the pure exploration problem in stochastic multi-armed bandits where the similarities between the arms is captured by a graph and the rewards may be represented as a smooth signal on this graph. In particular, we consider the problem of finding the arm with the maximum reward (i.e., the maximizing problem) or one that has sufficiently high reward (i.e., the satisficing problem) under this model. We propose novel algorithms \textbf{\algoname{}} (GRaph based UcB) and $\zeta$-\textbf{\algoname{}} for these problems and provide theoretical characterization of their performance which specifically elicits the benefit of the graph side information. We also prove a lower bound on the data requirement that shows a large class of problems where these algorithms are near-optimal. We complement our theory with experimental results that show the benefit of capitalizing on such side information.

\end{abstract}

\section{Introduction}
\label{sec:intro}
The multi-armed bandit has emerged as an important paradigm for modeling sequential decision making and learning under uncertainty. Practical applications include design policies for sequential experiments~\cite{robbins1952some}, combinatorial online leaning tasks~\cite{chen2014combinatorial}, collaborative learning on social media networks~\cite{ kolla2018collaborative, audibert2010best}, latency reduction in cloud systems~\cite{joshi2017efficient} and many others~\cite{ NIPS2015_ab233b68, zhou2014optimal, tekin2018multi, kandasamy2016multi}. In the traditional multi-armed bandit problem, the goal of the agent is to sequentially choose among a set of actions or arms to maximize a desired performance criterion or reward. This objective demands a delicate tradeoff between exploration (of new arms) and exploitation (of promising arms). An important variant of the reward maximization problem is the  identification of arms with the highest (or near-highest) expected reward. This \emph{best arm identification}~\cite{mannor2004sample, even2006action} problem, which is one of pure exploration, has a wide range of important applications like identifying and testing drugs to treat infectious diseases like COVID-19, finding relevant users to run targeted ad campaigns, hyperparameter optimization in neural networks and recommendation systems. The broad range of applications of this paradigm is unsurprising given its ability to essentially model any optimization problem of black-box functions on discrete (or discretizable) domains with noisy observations.



While pure exploration problems in bandits show considerable promise, there are significant hurdles to their practical usage. In modern applications, one is often faced with a tremendously large number of options (sometimes in the order of hundreds of millions) that need to be considered for decision making. In such cases, playing (i.e., obtaining a random sample from) each bandit arm  even once could be intractable. This renders traditional approaches to pure exploration ineffective. Fortunately, in several applications, the arms and their rewards are related to each other and information about the reward of one arm may be deduced from plays of similar arms. In this paper, we consider the pure exploration problem in stochastic multi-armed bandits where the similarities between arms is captured by a graph and the rewards may be represented as a smooth signal on this graph. Such graph side information is available in a wide range of applications: search and recommendation systems have graphs that capture similarities between items~\citep{guo2010structured, rao2015collaborative, wu2020graph, dasarathy2017computational}; drugs, molecules and their interactions can be represented on a graph \citep{drkg2020}; targeted advertising considers users connected to each other in a social network \citep{jamali2009trustwalker}, and hyperparameters for training neural network are often inter-related \citep{young2018hyperspace}. It is worth noting that such graphs are sometimes intrinsic to the problem (e.g., spatial coordinates or social/computer networks), or may be inferred based on a similarity metrics defined on arm features; a recent line of work considers constructing such graphs to enable more effective learning \citep[see e.g.,][]{Zhang2022GALAXY,kushnir2020Diffusion}. 


\paragraph{Our Contributions:}
\label{sec:contribs}
We consider the pure exploration in multi-arm bandits problem when a graph that captures similarities between the arms is available. In particular, we consider the problem of finding the arm with the maximum reward (i.e., the maximizing problem) or one that has sufficiently high reward (i.e., the satisficing problem\footnote{named after  Herbert Simon's celebrated alternative model of decision making \cite{simon1955behavioral}}) under the assumption that  arm rewards are smooth with respect to a known graph. Our main contributions may be summarized as follows: 

{\bf (a)} We devise a novel algorithm \algoname{} for the best arm identification problem (i.e., the maximizing problem) that specifically exploits the {\em homophily} (strong connections imply similar average rewards) on the graph (Section~\ref{sec:sampling_algo}).

\vspace{-1mm}
{\bf (b)} We provide a theoretical characterization of the performance of \algoname{}. To this end, we define a novel measure $\mathfrak{I}$ that we dub the ``\textit{influence factor}'' which depends on the resistance distance of the underlying graph. This measure captures the benefit of the graph side information and plays a central role in the analysis of \algoname{}.
In the traditional (graph-free) best arm identification problem, the sample complexity is know to scale as $\sum_{i = 1}^n \frac{1}{\Delta_i^2}$, where $\Delta_i$ is the gap between the expected rewards of the best arm and arm $i$. On the other hand, we show that \algoname{} roughly has a complexity that scales like $\sum_{i \in \mathcal{H}} \frac{1}{\Delta_i^2}$ samples where the set $\mathcal{H}$ is a set dependent on the influence factor, which contains arms which are hard to distinguish from optimal arm. For a broad range of problems $|\mathcal{H}| \ll  n$, yielding significant improvement over traditional best arm identification algorithms (Section \ref{sec:sample_complexity_base}). 

\vspace{-1mm}
{\bf(c)} In Section~\ref{sec:lower-bounds}, we provide lower bounds on the minimum number of samples required for identification of the optimal arm when a graph encoding arm similarities is available. This shows the near-optimality  of \algoname{} for an important class of representative problems.

\vspace{-1mm}
{\bf(d)} In many real world scenarios, the aim of finding the absolute best arm can often be too costly or even intractable. In these situations, it may be more appropriate to solve the \emph{satisficing} problem, where the algorithm returns an arm that is good enough. We propose a variant of \algoname{}, dubbed $\zeta$-\algoname{} for this important setting in Section~\ref{sec:zeta_best_arm}

\vspace{-1mm}
{\bf(e)} Finally, in Section~\ref{sec:exp}, we complement our theoretical results with an empirical evaluation of our algorithms. We further provide algorithmic improvements to \algoname{} and discuss novel sampling policies for best arm identification in the presence of graph information.

\subsection{Related Work}
\label{sec:related}


The textbook~\cite{lattimore2020bandit} is an excellent resource for the general problem of multi-armed bandits. The pure exploration variant of the bandit problem is more recent, and has also received considerable attention in the literature~ \cite{bubeck2009pure, bubeck2011pure, garivier2019nonasymptotic, gabillon2012best, audibert2010best, jamieson2014best}. These lines of work treat the bandit arms or actions as independent entities and playing a particular arm yields no information about any other arm. This leads to great difficulty in scaling such methods, since in the problem setups with large number of arms, attempting to play \textit{all} arms is not practical. We resolve this precise roadblock by introducing a convenient way of of appending graph side information into the mix which provably accelerates the process of sub-optimal arm elimination (potentially without playing it even once!)

A recent line of work ~\cite{li2016collaborative, lattimore2020bandit, gupta2020multiarmed, yang2020laplacianregularized, gentile2014online, ma2015active} has proposed the leveraging of structural side-information for the multi-armed bandit problem for regret minimization. Such topology-based bandit methods work under the assumption that pulling an arm reveals information about other, correlated arms ~\cite{gupta2020multiarmed, shamir2011variant}, which help in developing better regret methods. Similarly, spectral bandits \cite{kocak2020best, yang2020laplacianregularized, valko2014spectral} assume user features are modelled as signals defined on an underlying graph, and use this to assist in learning. The works~\cite{pmlr-v162-atsidakou22a} and~\cite{Online2015Yifan} consider similar graph information models, albiet at a degraded level. The authors in~\cite{lejeune2020thresholding} use the graphs to improve the regret bounds in a thresholding bandit setting. Work revolving around spectral bandits utilize the \textit{spectrum} of the graph laplacian. In contrast, we focus on the \textit{combinatorial properties} of the graphs to devise algorithms and analyse them. Another line of work \cite{dasarathy2015s2, wang2019distance, lipor2018quantile,ma2013optimality}  considers search problems on graphs under a different model and there is an opportunity for future work to combine these techniques. 
 
Most of the aforementioned works focus on regret minimization in the presence of graph information. The problem of pure exploration with similarity graphs has received far less attention. The authors in \cite{kocak2020best} were the first to attempt at filling this gap  for the spectral bandit setting. They provide an information-theoretic lower bound and a gradient-based algorithm to estimate this lower bound to sample the arms. The authors provide performance guarantees for the algorithm, but these results only indirectly capture the benefit brought by the graph; our results on the other hand are based on a novel complexity measure that explicitly elicits the benefit of having the graph side information. 

Note that, similarity graph information considered in this work is fundamentally different from linear rewards assumption in contextual/linear bandits. In the linear bandits problem, the reward behavior is assumed to be low dimensional and this is crucial for the improved regret bounds and sample complexity guarantees~\cite{lattimore2020bandit, soare2014bestarm}. In the current work we do not make any assumptions on low dimensionality of the rewards but still show improvements in sample complexity provided a good arm-similarity graph is available. We show a toy example in Appendix~\ref{app:as_linear_bandit} where a low dimensional linear bandit cannot be competitive with the corresponding graph-bandit setting.

\section{Problem Setup and Notation}\label{sec:prob_intro}
We consider an $n$-armed bandit problem with the set of arms given by $[n] \triangleq \{1,2,3,\dots, n\}$. Each arm $i \in [n]$ is associated with a $\sigma$-sub-Gaussian distribution $\nu_i$. That is, $\mathbb{E}_{X\sim \nu_i} \left[ \exp\left( s(X - \mu_i) \right) \right] \leq \exp\left( \frac{\sigma^2 s^2}{2} \right)$ 
$\forall s \in \mathbb{R}$, where $\mu_i = \mathbb{E}_{\nu_i}\left[ X \right]$ is said to be the (expected or mean) reward associated to arm $i$. We will let $\pmb{\mu}\in \mathbb{R}^n$ denote  the vector of all the arm rewards. A ``play'' of an arm $i$ is simply an observation of an independent sample from $\nu_i$; this can be thought of as a noisy observation of the corresponding mean $\mu_i$. 
The goal of the best-arm identification problem is to identify, from such noisy samples, the arm $a^\ast\triangleq \arg\max_{i\in [n]} \mu_i$ that has the maximum expected reward, denoted by $\mu^\ast$. For each arm $i\in [n]$, we will let $\Delta_i \triangleq \mu^\ast - \mu_i$  denote the sub-optimality of the arm.

As discussed in Section~\ref{sec:intro}, our goal is to consider the best-arm identification where one has additional access to information about the similarity of the arms under consideration. In particular, we model this side information as a weighted undirected graph $G = (V_G, E_G, A_G)$ where the vertex set, $V_G = [n]$, is identified with the set of arms, the edge set $E_G \subseteq {[n] \choose 2}$, and adjacency matrix $A_G\in \mathbb{R}^{n\times n}$ describes the weights of the edges $E$ between the arms which captures the similarity in means of connected arms; the higher the weight, the more similar the rewards from the corresponding arms. We will let $L_G = D_G - A_G$ denote the combinatorial Laplacian\footnote{All our results continue to hold if this is replaced with the normalized, random walk, or generalized Laplacian.} of the graph \cite{chung1997spectral}, where $D_G = {\rm diag}(A_G\times \mathbb{1}_n)$ is a diagonal matrix containing the weighted degrees of the vertices. We will suppress the dependence on $G$ when the context is clear. 
Subsequently, we show that if one has access to this graph and the vector of rewards $\pmb{\mu}$ is \emph{smooth} with respect to the graph (that is, highly similar arms have highly similar rewards), then one can solve the pure exploration problem extremely efficiently.
We will capture the degree of smoothness of $\pmb{\mu}$ with respect to the graph using the following seminorm\footnote{$L_G$ is not positive definite, and can be verified to have as many zero eigenvalues as the number of connected components in $G$}: 
\begin{equation}
  \left\| \pmb{\mu} \right\|_G^2 \triangleq \langle\pmb{\mu}, L_G\pmb{\mu}\rangle = \sum_{\{i,j\}\in E_G}A_{ij} (\mu_i - \mu_j)^2.  
\end{equation}
The second equality above can be verified by a straightforward calculation. Also, notice that $\left\| \pmb{\mu} \right\|_G$ being small implies $\mu_i\approx \mu_j$ for $(i, j)\in E$. In such scenario we say that the mean vector $\pmb{\mu}$ is smooth over graph $G$. This observation has inspired the use of the Laplacian in several lines of work to enforce smoothness on the vertex-valued functions \cite{ando2007learning, valko2014spectral, zhu2005semi, lejeune2020thresholding}. For $\epsilon >0$, we say that arms (rewards) are $\epsilon$-smooth with respect to a graph $G$ if $\left\| \pmb{\mu} \right\|_G \leq \epsilon$. 

Let $\mathcal{C}(G) \subset 2^{[n]}$ denote the set of all connected components and let $k(G) \triangleq \left| \mathcal{C}(G) \right|$ denote the number of connected components of the graph $G$. For a vertex $i \in [n]$, we will let $C_i(G) \in \mathcal{C}(G)$ denote the connected component that contains $i$. When the context is clear we sometimes let $C_i(G)$ also refer all the nodes in the connected component. We say a graph $G = ([n],E)$ has $k$-\textit{isolated cliques} if it can be divided into fully connected sub-graphs $G_i = (V_i,E_i)$ such that $V_i \subseteq [n], E_i = {V_i \choose 2}$ for all $i \in [k], V_i\cap V_j = \emptyset, E_i\cap E_j = \emptyset $ for all $i,j \in [k]$, and $\bigcup_{i=1}^k V_i = [n],\bigcup_{i=1}^k E_i = E$. Notice that we only have one clique if $G$ is fully connected.


To solve the best-arm identification problem, we need a sampling policy to sequentially and interactively select the next arm to play, and a stopping criterion. For any time $t\in \mathbb{N}$, the sampling policy $\pmb{\pi}_t = \{\pi_s\}_{s\leq t}$ is a function that maps $t$ to an arm in $[n]$ given the history of observations up to time $t-1$. With slight abuse of notation, we will let $\pi_t$ denote the arm chosen by an agent at time $t$. Let $r_{t, \pi_t}$ denote the random reward observed at time $t$ from arm $\pi_t$. We use $t_i(\pmb{\pi}_t)$ (referred as $t_i$ for simplicity) to denote the number of times arm $i$ is played under the sampling policy $\pmb{\pi}_t$. In this paper we tackle the following problems:
\textit{\paragraph{P1 (Best arm identification):}Given $n$ arms and an arbitrary graph $G$ capturing similarity between the arms, can we design a policy $\pmb{\pi}_T$ that exploits the similarity  to find the best arm efficiently?}
\textit{\paragraph{P2 ($\zeta$-best arm identification):}Under the setting in \textbf{P1}, can we design a similarity exploiting policy $\pmb{\pi}_T$ so as to find an arm belonging to the set $B(\zeta)\triangleq \{i \in [n] : | \mu_i - \mu_{a^\ast} | \leq \zeta \}$ efficiently?}

\section{The \textbf{\algoname} Algorithm}\label{sec:sampling_algo}


We now introduce \algoname{} (GRaph based Upper Confidence Bound), a novel but natural algorithm for best arm identification in the presence of graph side information. We begin with an intuitive description of how \algoname{} incorporates the graph side information into an \textit{upper confidence bound} (UCB) strategy. 
Most UCB algorithms~\cite{lattimore2020bandit, valko2014spectral} compute the estimates of mean and variance, and use these to eliminate arms that have been deduced to be sub-optima.  The key idea behind \algoname{} is that the arm similarity information allows us to create high-quality estimates of mean rewards and confidence intervals for arms that have not been (sufficiently) sampled yet. In what follows, we describe the building blocks of \algoname{}.





\subsection{Leveraging Graph Side Information}\label{sec:arm_elim}

We introduce two key ideas that lie at the heart of the \algoname{} algorithm. First, at each step, \algoname{} computes a regularized estimate of the means of \emph{all the arms}; the regularization based on the graph Laplacian essentially promotes the smoothness of the mean vector on the given graph. This allows the algorithm to estimate means of arms it has {\em never sampled}. 
To do this, at any given time step $T$, the algorithm solves the following  Laplacian-regularized least-squares optimization program: 
\begin{align}\label{eq:mu_estimate}
    \hat{\pmb{\mu}}_T = \mathop{\arg\min}_{\pmb{\mu}\in \mathbb{R}^{n}}~\left\{~ \left[\sum_{t=1}^T (r_{t, \pi_t} - \mu_{\pi_t})^2 \right] + \rho\langle\pmb{\mu}, L_G\pmb{\mu}\rangle\right\},
\end{align}
where $\rho >0$ is a tunable parameter.  
Equation \eqref{eq:mu_estimate} admits a closed form solution of the form 
$$\hat{\pmb{\mu}}_T =  \left(\sum_{t=1}^T\mathbf{e}_{\pi_t}\mathbf{e}_{\pi_t}^\top +\rho L_G\right)^{-1}\left(\sum_{t=1}^T \mathbf{e}_{\pi_t}r_{t, \pi_t}\right),$$
provided the matrix $V_T \triangleq \sum_{t=1}^T\mathbf{e}_{\pi_t}\mathbf{e}_{\pi_t}^\top +\rho L_G$ is invertible; 
$\mathbf{e}_i$ denotes the $i$-th standard basis vector for the Euclidean space $\mathbb{R}^{n}$. In Appendix~\ref{app:support} we show that invertibility holds if and only if the sampling policy yields at least one sample per connected component of $G$. This is a rather mild condition that we arrange for explicitly in our algorithm, given that we know the graph $G$. In what follows we assume that every connected component of graph $G$ is sampled at least once. This regularized mean estimation procedure yields an estimate of the mean that is both in agreement with observations and smooth on the graph -- thereby allowing information sharing among similar arms.




The second key idea of our algorithm is the utilization of the graph $G$ in tracking the confidence bounds of {\em all the arms simultaneously}. Intuitively, for identifying the best arm, we must be reasonably certain about the sub-optimality of the other arms. This in turn would require the algorithm to track a high-probability confidence bound on the means of all the arms. In the traditional (graph-free) best arm identification problem, the confidence interval of an arm's mean estimate depends on the number of times the arm has been played. Requiring multiple plays of all suboptimal arms for obtaining high confidence bounds is potentially disastrous when the number of arms is very large. In our setup, we show that the knowledge of the similarity graph greatly improves this situation. In particular, we show that a play of any arm not only tightens its own confidence interval, but also has an impact on the confidence intervals of {\em all connected arms}. To quantify the benefit of graph information for the confidence bounds, we will define a novel quantity for each arm -- the effective number of plays. 
%
%
%
%
\begin{definition}[Effective Number of Plays]\label{def:effective_samples}
Let $\rho>0$ and $\{t_i\}_{i=1}^n$ denote the number of plays of each of the $n$ arms when a sampling policy $\pmb{\pi}_T$ is employed for $T$ time steps. Suppose that for each connected component $C \in \mathcal{C}(G)$, there is at least one arm $i_C\in C$ such that $t_{i_C}>0$. Then the effective number of plays for each arm $i\in [n]$ is defined as
   \(t_{\text{eff}, i} \triangleq \left[\left(N_T+\rho L_G\right)^{-1}\right]_{ii}^{-1}\),
where $N_T$ is a diagonal matrix of $\{t_i\}_{i=1}^n$, and $L_G$ denotes the Laplacian of the given graph $G$.
\end{definition}
Effective number of plays $t_{\text{eff}, i}$ for any arm $i$ is influenced by two factors: (a) the number of samples of arm $i$ itself, and (b) the number of samples of any arm in the connected component $j\in C(i), j\neq i$. It can be shown that for any arm $i$, $t_{\text{eff}, i}$  depends on the number of connections of node $i$ in graph $G$ and its value increases as the connectivity of the node increases. The choice of the terminology for this quantity is justified by the following lemma, which provides a high confidence bound for the mean estimate of each arm . 
\begin{lemma}[Concentration inequality]\label{lem:variance_estimate}
For any $T>k(G)$, the following holds with probability at least $1-\delta$:
\begin{align}\label{eq:hp_bounds_proof}
    |\hat{\mu}^{i}_T - \mu_{i}| \leq \sqrt{\frac{1}{t_{\text{eff}, i}}}\left(2\sigma\sqrt{14\log{\left(\frac{2w_i(\pmb{\pi}_T)}{\delta}\right)}} + \rho\|\pmb{\mu}\|_G\right),~~~\forall i\in [n]
\end{align}
where $w_i(\pmb{\pi}_T) = a_0nt_{\text{eff}, i}^2$ for any constant $a_0>0$, $\hat{\mu}^i_T$ is the $i$-th coordinate of the estimate from \eqref{eq:mu_estimate} 
\end{lemma}

Notice that the \textit{effective number of plays} has a similar role as the number of plays in traditional pure exploration algorithms~\cite{even2006action}. Indeed, in the absence of graph information, $t_{{\rm eff},i}$ reduces to $t_i$, the total number of plays of individual arms. Lemma~\ref{lem:variance_estimate} recovers high confidence bounds for standard best-arm identification problem~\cite{even2006action}. It should be noted that while our work is the first to identify this interpretable quantity explicitly, the result of Lemma~\ref{lem:variance_estimate} in other forms has appeared before in the literature~\cite{NIPS2011_e1d5be1c, valko2014spectral, yang2020laplacianregularized}. 



We introduce our algorithm \algoname{} for best arm identification when the arms can be approximately cast as nodes on a graph. \algoname{} uses insights from graph-based mean estimation~\eqref{eq:mu_estimate} and upper confidence bound estimation~\eqref{eq:hp_bounds_proof} for its elimination policies to search for the optimal arm. 

\algoname{} accepts as input a graph $G$ on $n$ arms (and its Laplacian $L_G$), a regularization parameter $\rho >0$, a smoothness parameter  $\epsilon > 0$, and an error tolerance parameter $\delta \in (0,1)$. It  is composed of the following major blocks.\\
{\bf Initialization: }
First, \algoname{} identifies the clusters in the $G$ using a \texttt{Cluster-Identification} routine. Any algorithm that can efficiently partition a graph can be used here, e.g METIS~\cite{Karypis98afast}.
\algoname{} then samples one arm from each cluster. This ensures $V_T \succ 0$, which enables \algoname{} to estimate $\hat{\pmb{\mu}}_T$ using the closed form solution of  eq.~\eqref{eq:mu_estimate}. A great advantage of \algoname{} is that the initialization phase only requires steps equal to the number of disconnected components in the graph. This is in direct contrast with traditional best arm identification algorithms, which require atleast one sample from every arm initially.\\
%
%
%
%
%
%
%
%
{\bf Sampling policy: } At each round, \algoname{} obtains a sample from the arm returned by the routine {\tt Sampling-Policy}, which cyclically samples arms from different clusters while ensuring that no arm is resampled before all arms in consideration have the same number of samples. This is distinct from standard cyclic sampling policies that is traditionally used for best arm identification~\cite{even2006action}, but any of them may be modified readily to provide a cluster-aware sampling policy for \algoname{}. In our experiments, we show that replacing cyclic sampling with more statistics- and structure-aware sampling greatly improves performance; a theoretical analysis of these is a promising avenue for future work. One of the major advantage of \algoname{} is the lite nature of the computation. Every loop just requires a rank-1 inverse update which can be performed very efficiently and it does not need any subroutines, unlike~\citep{kocak2020best} \\
{\bf Bad arm elimination : } 
At any time $t$, let $A$ be the set of all arms in consideration for being optimal. Using the uncertainty bound from~\eqref{eq:hp_bounds_proof}, \algoname{} uses the following criteria for sub-optimal arm elimination.
%
%
At each iteration, \algoname{} identifies an arm $a_{\max}\in A$, $a_{\max} = \underset{i\in A}{\arg\max} \left[\hat{\mu}^i_{t} - \beta_i(t)\sqrt{t_{\text{eff},i}^{-1}}\right]$, 
where $\beta_i(t)= \left(2\sigma\sqrt{14\log{\left(\frac{2na_0t_{\text{eff}, i}^2}{\delta}\right)}} + \rho\epsilon\right)$,  with the \textit{highest lower bound} on its mean estimate.
   Following this, \algoname{} removes arms from the set $A$ according to the following elimination policy,
\begin{align}\label{eq:elimination_routine}
    A\leftarrow \left\{\mathbf{a}\in A~ | ~\hat{\mu}^{a_{\max}}_{t} - \hat{\mu}^a_{t} \leq \beta_a(t)\sqrt{t_{\text{eff},a}^{-1}} + \beta_{a_{\max}}(t)\sqrt{t_{\text{eff},a_{\max}}^{-1}}\right\}.
\end{align}




Note that~\algoname{} does not require any optimization innerloop as in~\cite{kocak2020best}. This potentially provides \algoname{} with a significant computation advantage, especially when the dimensionality of the problem is very large. The pseudocode for \algoname{} can be found in Appendix~\ref{app:generic_sample_complexity}.

\begin{algorithm}[H]
            \caption{\algoname}\label{alg:GRUB}
            \begin{algorithmic}
                \STATE {\bfseries Input:} Regularization parameter $\rho$, Smoothness parameter $\epsilon$, Error bound $\delta$, Total arms $n$, Laplacian $L_G$, Sub-gaussianity parameter $\sigma$
                \STATE $t\leftarrow 0$\;
                \STATE $A = \{1,2,\dots, n\}$\;
                \STATE $t=0$\;
                \STATE $V_0 \leftarrow \rho L_G$\;
                \STATE $\mathcal{C}(G) \leftarrow $ {\tt Cluster-Identification}($L_G$)\;
                \FOR{$C\in \mathcal{C}(G)$}
                    \STATE $t\leftarrow t+1$\;
                    \STATE Pick random arm $k \in C$ to observe reward $r_{t,k}$\;
                    \STATE $V_{t}\leftarrow V_{t-1}+\mathbf{e}_k\mathbf{e}^T_k$, and $\mathbf{x}_{t}\leftarrow \mathbf{x}_{t-1} + r_{t, k}\mathbf{e}_k$\;
                \ENDFOR
                \WHILE{$|A| > 1$}
                    \STATE $t\leftarrow t+1$\;
                    \FOR{$i\in A$}
                        \STATE $t_{\text{eff}, i} \leftarrow ([V_{t}^{-1}]_{ii})^{-1}$
                        \STATE $\beta_i(t) \leftarrow 2\sigma\sqrt{14\log{\left(\frac{2nt_{\text{eff}, i}^2}{\delta}\right)}} + \rho\epsilon$\;
                    \ENDFOR
                    \STATE $k\leftarrow $ {\tt Sampling-Policy}($t, V_t, A$, $\mathcal{C}(G)$)\;
                    \STATE Sample arm $k$ to observe reward $r_{t,k}$\;
                    \STATE $V_{t}\leftarrow V_{t-1}+\mathbf{e}_k\mathbf{e}^T_k$\;
                    \STATE $\mathbf{x}_{t}\leftarrow \mathbf{x}_{t-1} + r_{t,k}\mathbf{e}_k$\;
                    \STATE $\hat{\pmb{\mu}}_{t}\leftarrow V_{t}^{-1}\mathbf{x}_{t}$\;
                    \STATE $a_{\max} \leftarrow \underset{i\in A}{\arg\max} \left[\hat{\mu}^i_{t} - \beta(t)\sqrt{t_{\text{eff},i}^{-1}}\right]$\;
                    \STATE $A\leftarrow \left\{\mathbf{a}\in A~ | ~\hat{\mu}^{a_{\max}}_{t} - \hat{\mu}^a_{t} \leq  \beta_a(t)\sqrt{t_{\text{eff},a}^{-1}}\right.$\;
                    \STATE $ \left. \qquad\qquad + \beta_{a_{\max}}(t)\sqrt{t_{\text{eff},a_{\max}}^{-1}} \right\}$\;
                \ENDWHILE
                \RETURN A
            \end{algorithmic}
        \end{algorithm}


Next, we derive performance guarantees on the sample complexity for \algoname{} to return the best arm with high probability. 

\section{Theoretical Analysis of \algoname{}}\label{sec:sample_complexity_base}
In this section we provide a formal statement of the sample complexity of \algoname{}. To do this, we first introduce a novel quantity  we call {\em influence factor}. The influence factor of an arm is derived from resistance distance, a classical graph theoretic concept. This adds to the interpretability and understanding of the instances where using graph side information might be of tremendous use to the application. The usage of graph through the influence factor allows us to identify arms that can be eliminated quickly from consideration. 

\subsection{Resistance Distance and Influence Factor}

We first recall the definition of resistance distance in a graph.

\begin{definition}[Resistance Distance]\label{def:res_dis}\cite{bapat2010resistance}
For any graph $G$ with $n$ nodes, given a constant $\delta >0$, the \textbf{resistance distance} $r_{\delta, G}(i,j)$ between two nodes $i, j$ is defined as,
\begin{align}
    r_{\delta, G}(i, j) = R_{ii} + R_{jj} - R_{ij} - R_{ji},
\end{align}
where $R \triangleq \left(L_G + \delta\mathbbm{1}\mathbbm{1}^T\right)^\dagger$; $\dagger$ denotes the Moore-Penrose inverse, $L_G$ is the Laplacian of graph $G$, and $\mathbbm{1}\in \mathbb{R}^n$ is the vector of all 1's.
\end{definition}

When the context is clear we denote the resistance distance simply as $r_G(\cdot, \cdot)$. The terminology comes from circuit theory: Suppose that an graph $G = ([n], E)$ is thought of as a resistor network on the nodes $[n]$ where each edge $\{i,j\}$ has a unit resistance. Then, the effective resistance between two nodes $i$ and $j$ is precisely the resistance distance $r(i,j)$. 
%
It can be shown in general that nodes that are close by or connected by several paths have a small resistance distance. Given its ability to capture closeness of nodes in graph, the resistance distance has found a broad range of applications and has been the subject of much study; see e.g., \citep{klein1993resistance, bapat2010resistance, xiao2003resistance}. 


Using the notion of resistance distance, we define the influence factor $\mathfrak{I}(\cdot, G)$ of a vertex below. This novel measure quantifies the impact of the graph on the parameter estimation of arm $j$, and in particular, allows us to use the combinatorial properties of the graph and the arm means to classify arms into two sets: competitive and non-competitive; the definition of these sets follows right after. As our theory  will show, the competitive arms are sampled as though we were in the traditional graph-free setting; on the other hand, non-competitive arms are eliminated rapidly, often with zero plays! Indeed, the smoother the reward vector is with respect to the graph, the fewer competitive arms there are -- it is this phenomenon that is captured using the influence factor. 
\begin{definition}[Influence Factor]\label{def:d_better}
Let $G$ be a graph on the vertex set $[n]$. For each $j\in [n]$, define \textbf{influence factor $\mathfrak{I}(j, G)$} as:
\begin{align}\label{eq:d}
     \mathfrak{I}(j, G) = \begin{cases} \underset{i\in C_j(G), i\neq j}{\min}\{r_G(i,j)^{-1}\},&\ \text{ if }~~|C_j(G)|>1,\\
    0, ~~~&\ \text{ otherwise }.
    \end{cases}
\end{align}
Here, $r_G(i,j)$ is the resistance distance between arm $i$ and $j$ in $G$ as in Definition~\ref{def:res_dis}. 
\end{definition}




\begin{definition}[Competitive and Non-Competitive Arms]\label{def:b_d_set_def}
Fix $\pmb{\mu}\in \mathbb{R}^n$, graph $D$, regularization parameter $\rho$, confidence parameter $\delta$, and smoothness parameter $\epsilon$. We define $\mathcal{H}_D$ to be the set of competitive arms and $\mathcal{N}_D$ to be the set of non-competitive arms as follows: 
%
%
\begin{align}\label{def:b_i_d_i}
    \mathcal{H}_D =  \left\{j\in [n] \big| \Delta_i \leq 2\sqrt{\frac{2}{\rho\mathfrak{I}(i, D)}} \left(2\sigma\sqrt{14\log{\left(\frac{2a_0n\rho^2\mathfrak{I}(i, D)^2}{\delta}\right)}} + \rho\epsilon\right)\right\}
\end{align}
and $\mathcal{N}_D \triangleq [n]\setminus \mathcal{H}_D$. 
\end{definition}
As the name suggests, the arms in $\mathcal{H}$ are close to the optimal arm $a^*$ in mean (competitive compared to the optimal arm $a^\ast$) and requires several plays before they can be discarded, as shown in the theorem below. Note from the above definition that an arm is more likely to be part of this set if its mean is high (i.e., $\Delta_i$ is low) and its influence factor is low. Similarly, the non-competitive set is composed of arms whose means are not competitive with the optimal arm. 

Armed with these definitions, we are now ready to state our main theorem that characterizes the performance of \algoname{}.

\subsection{Sampling policy performance}\label{sec:sampling_}
Cyclic sampling policies have been traditionally used in multi-armed bandit problems for best-arm identification~\cite{even2006action}. The sample complexity bound for \algoname{} with cyclic sampling is as follows:


\begin{theorem}[GRUB Sample Complexity]\label{thm: sample_complexity}
Consider $n$-armed bandit problem with mean vector $\pmb{\mu}\in \mathbb{R}^n$. Let $G = (V,E)$ be the similarity graph with the vertex set $V = [n]$ and edge set $E$, let $\mathcal{G}$ be the set of subgraphs of $G$ , and further suppose that $\pmb{\mu}$ is $\epsilon$-smooth i.e., $\|\pmb{\mu}\|_G \leq \epsilon$. Define 
\begin{align*}\label{eq:T_sample_generic_base}
    T_{\text{sufficient}} & \triangleq  \mathop{\arg\min}_{D\in \mathcal{G}} \sum_{C\in\mathcal{C}_D}\left[\sum_{\substack{i\in C\cap\mathcal{H}_D\\i\neq 1}} \frac{1}{\Delta_i^2}\left[c_1\log{\frac{c_2}{\delta\Delta_i}} + \frac{\rho\epsilon}{2}\right]  + \max_{i\in C\cap\mathcal{N}_D}\frac{2}{\Delta_i^2}\left[c_1\log{\frac{c_2}{\delta\Delta_i}} + \frac{\rho\epsilon}{2}\right]\right], 
\end{align*}
where $\Delta_i = \mu^*-\mu_i$ for all suboptimal arms, $\mathcal{H}_D$ and $\mathcal{N}_D$ are as in Definition~\ref{def:b_d_set_def}, $\mathcal{C}_D$ is the set of connected components of a given graph $D$ and $c_1, c_2$ are constants independent of system parameters. Then, with probability at least $1-\delta$, \algoname{}: (a) terminates in no more than $T_{\text{sufficient}}$ rounds, and (b) returns the best arm $a^\ast = \arg\max_i \mu_i$. 
\end{theorem}

\begin{remark} The required number of samples for successful elimination of suboptimal arms, and therefore the successful identification of the best arm, can be split into two categories based on the sets defined in Definition~\ref{def:b_d_set_def}. Each sub-optimal {\em highly competitive arm} $j\in \mathcal{H}$ requires $\mathcal{O}(1/\Delta_j^2)$ samples, which is comparable to the classical (graph-free) best-arm identification problem. Additionally, the non-competitive arms $\mathcal{N}$ can be eliminated without being played, depending on the influence factor: one round of the cyclic sampling suffices to eliminate these arms (even if they are never played!). We refer the reader to Appendix~\ref{app:generic_sample_complexity} for a more detailed discussion. Indeed, the smaller $|\mathcal{H}|$ is, the more the graph side information benefits \algoname and vice-versa. 
\end{remark}

\begin{remark} Note that $T_{\text{sufficient}}$ in Theorem~\ref{thm: sample_complexity} involves the minimum over all subgraphs. As we show in Lemma~\ref{lem:d_change_subgraph} in the appendix,  $\mathfrak{I}$ can actually increase if one restricts their attention to certain subgraphs of $G$; this in turn increases the size of $\mathcal{N}$ and decreases the size of  $\mathcal{H}$, hence, giving a tighter upper bound on the performance of the algorithm. \algoname{} {\it  automatically adapts to the best subgraph} to maximize the influence factor $\mathfrak{I}(\cdot, \cdot)$ to obtain the best possible sample complexity and this is reflected in the statement of Theorem~\ref{thm: sample_complexity}.


\end{remark}

The complete proof of Theorem~\ref{thm: sample_complexity} can be found in Appendix~\ref{app:generic_sample_complexity}, where we also provide more insights on the behavior of the confidence bound as a function of the number of samples acquired. These results may be of independent interest to the reader.

\subsection{Improved sampling policies}\label{subsec:sampling_policies}

As can be inferred from the psuedocode of Algorithm~\ref{alg:GRUB} the primary goal of the {\tt Sampling-Policy} is the quick and safe elimination of suboptimal arms, achieved through shrinking of the confidence bounds $\beta_i(t)\sqrt{(t_{\text{eff}, i})^{-1}}$ for all arms $i$ still in consideration at time $t$. 

Theorem~\ref{thm: sample_complexity} established guarantees on $T_{\text{sufficient}}$ for naive cyclic sampling policy, i.e. a sampling policy which doesn't directly exploit the graph properties in this arm choice. Note that, even if the sampling policy doesn't utilize any graph properties, the similarity graph is still being utilized in computing the mean estimate and the confidence widths. To enhance the involvement of graph structural information in arm sampling policy, a few alternatives can be characterized as:
\begin{itemize}
    \item {\bf Marginal variance minimization (MVM):} Pick the arm which has the highest confidence bound width. Specifically, at time $t$, let $\pi_T = \underset{i\in A}{\arg\min}~ t_{\text{eff}, i} = \underset{i\in A}{\arg\max} ~[V_{T}^{-1}]_{ii}$, where $A$ is the set of indices of the arms under consideration. \\
    \item {\bf Joint variance minimization -- nuclear (JVM-N): }This variant is inspired from the concept of V-optimality~\cite{pmlr-v22-ji12}. JVM-N picks the arms which leads to maximum decrease in the value of confidence widths across all arms, in the sense of nuclear norm. Specifically, $\pi_T = \underset{i\in A}{\arg\min} \|(V_{T} + \mathbf{e}_i\mathbf{e}_i^T)^{-1}\|_{*} - \|V_T^{-1}\|_{*}$, where  $\|\cdot\|_{*}$ denotes the nuclear norm.\\
    \item {\bf Joint variance minimization -- operator (JVM-O). }Taking inspiration from $\Sigma$-optimality~\cite{10.5555/3020847.3020904, ma2013sigma}, JVM-O picks arms which leads to maximum decrease in the value of confidence widths across all arms  in the sense of operator norm.
  $\pi_T = \underset{i\in A}{\arg\min} \|(V_{T} + \mathbf{e}_i\mathbf{e}_i^T)^{-1}\|_{\text{op}} - \|V_T^{-1}\|_{\text{op}}$ 
  
\end{itemize}


Comparison of the performance of MVM, JVM-N and JVM-O with the baseline of cyclic sampling is provided through synthetic experiments in Section~\ref{sec:exp}. In the next section, we derive fundamental lower bounds on the sample complexity \textit{any} algorithm requires in order to solve the said problem

\color{black}
\section{Lower Bounds}\label{sec:lower-bounds}
Let us consider an $n$-armed bandit setup with arm indices $[1, \dots, n]$. Let $\mu^*$ indicate the mean of the optimal arm and $\mu_i$ indicate the mean values of all other arms such that $\mu_i < \mu^*$. For the rest of this section, without loss of generality, let the index of optimal arm be 1. 



\begin{theorem}\label{thm:lower-bound}
Given an $n$-armed bandit model with associated mean vector $\pmb{\mu}\in \mathbb{R}^n$ and similarity graph $G$ smooth on $\pmb{\mu}$, i.e. $\langle\pmb{\mu}, L_G\pmb{\mu}\rangle \leq \epsilon$, for any $0<\epsilon <\epsilon_0$. Let $G = ([n],E)$ be the graph with only isolated cliques and w.l.o.g let arm 1 be the optimal arm. Then define
\begin{equation}
    T_{\text{necessary}} = \sum_{C\in \mathcal{C}_G/{C^*}}\min_{j\in C} \left\{\frac{4\sigma^2\log 5}{(\Delta_j - \sqrt{\epsilon})^2}\right\} + \sum_{j\in C^*/1 } \frac{4\sigma^2\log 5}{\Delta_j ^2}, 
\end{equation}
where $C^*$ is the clique with the optimal arm and $\epsilon_0 := \underset{i\in [n]/1, j\in C(i)}{\min}\left[\Delta_j \left[1 - \frac{\Delta_i}{\sqrt{\Delta_i^2 + \Delta_j^2}}\right]\right]^2$. Then any $\delta$-PAC algorithm will need at-least $T_{\text{necessary}}$ steps to terminate, provided $\delta \leq 0.1$.
\end{theorem}




Using Theorem~\ref{thm:lower-bound}, we can show that \algoname{} is minimax optimal for a $n$-armed bandit problems for certain class of similarity graph $G$. The following result shows that the upperbound on the sample complexity provided in Theorem~\ref{thm: sample_complexity} matches the lower bound established in Theorem~\ref{thm:lower-bound} in $\Delta_i$ up to a constant factor.

\begin{corollary}[Isolated clusters]\label{cor:lower-bounds-clusters}
Consider the setup as in Theorem~\ref{thm:lower-bound} with the further restriction that graph $G$ be such that the optimal node is isolated and $\epsilon < \min_{j\in [n]}\frac{\Delta^2_j}{2}$. Define,
\begin{align}
    T_{\text{necessary}}\geq \sum_{C\in \mathcal{C}_G/\{1\}} \max_{j\in C}\left\{ \frac{8\sigma^2 \log 5}{\Delta_j^2}\right\}.
\end{align}
Then any algorithm that takes fewer than $T_{\text{necessary}}$ samples will have a probability of error at least $0.1$. 
\end{corollary}

As can be seen in Corollary~\ref{cor:lower-bounds-clusters}, the lower bound expression can scale as standard $n$-armed bandit (implying no added advantage of having graph side-information) or can behave as a $|\mathcal{C}_G|$-armed bandit problem (scales as the number of clusters in graph $G$ rather than number of nodes $n$) purely by changing the similarity graph $G$. The difference between $\mathcal{C}_F$ (connected components in the subgraph constructed by making optimal arm isolated) and $\mathcal{C}_G$ (connected components in the given similarity graph) can lead to more interesting behaviour in terms of lower bound expressions on sample complexity.

\section{$\zeta$-best-arm identification}\label{sec:zeta_best_arm}

It can be observed from Theorem~\ref{thm: sample_complexity} that the fact that the means are $\epsilon$-smooth implies that distinguishing arm $j$ from $a^\ast$ would require at least $O(\epsilon^{-2})$ samples. A tighter upper bound on the violation $\epsilon$ and an edge between $j$ and $a^*$ would make the suboptimal arm $j$ harder to eliminate. However, it stands to reason that in such situations, it might be more practical to not demand for the absolute best arm, but rather an arm that is nearly optimal.  Indeed, in several modern applications we discuss in Section~\ref{sec:intro}, finding an approximate best arm is tantamount to solving the problem. In such cases, a simple modification of \algoname{} can be used to quickly eliminate definitely suboptimal arms, and then output an arm that is guaranteed to be nearly optimal. To formalize this, we consider the $\zeta$-best arm identification problem as follows. 




\begin{definition}\label{def:zeta_best_arm}
For a given $\zeta >0$, arm $i$ is called $\zeta$-best arm if $\mu_i \geq \mu_{a^*} - \zeta$, where $a^\ast = \arg\max_i \mu_i$ 
\end{definition}

The goal of the $\zeta$-best arm identification problem is to return an arm $\tilde{a}$ that is $\zeta-$optimal. We achieve this by a simple modification to \algoname{}, which we dub $\zeta-$\algoname{}, which ensures that all the remaining arms $i$ satisfy  $4\beta(t_{i})\sqrt{t_{\text{eff}, i}^{-1}} \leq \zeta$. It then outputs the best arm amongst those that are remaining. The following theorem characterizes the sample complexity for $\zeta$-\algoname{}:
\begin{theorem}\label{thm: disc_zeta_sc_v2} 
Consider $n$-armed bandit problem with mean vector $\pmb{\mu}\in \mathbb{R}^n$. Let $G$ be the given similarity graph on vertex set $[n]$, and further suppose that $\pmb{\mu}$ is $\epsilon$-smooth. Let $\mathcal{C}$ be the set of connected components of $G$. Define,
\begin{align}\label{eq:zeta_T_sample_generic_base}
    T_{\text{sufficient}} \triangleq & \mathop{\arg\min}_{D\in \mathcal{G}} \sum_{C\in\mathcal{C}_D}\left[\sum_{i\in C\cap\mathcal{H}_D} \frac{1}{(\Delta_i\vee\zeta)^2}\left[c_1\log{\frac{c_2}{\delta(\Delta_i\vee\zeta)}} +\frac{\rho\epsilon}{2}\right] \right. \nonumber\\
    + &  \left.\max_{i\in C\cap\mathcal{N}_D} \left\{ \frac{2}{(\Delta_i\vee\zeta)^2} \left[c_1\log{\frac{c_2}{\delta(\Delta_i\vee\zeta)}} +\frac{\rho\epsilon}{2}\right]\right\}\right],
\end{align}
where $\Delta_i = \mu^*-\mu_i$ for all suboptimal arms, $\mathcal{H}_D$ and $\mathcal{N}_D$ are as in Definition~\ref{def:b_d_set_def}, $\mathcal{C}_D$ is the set of connected components of a given graph $D$and $\Delta_i\vee{\zeta} = \max\{\zeta, \Delta_i\}$ and $c_1, c_2$ are constants independent of system parameters. Then, with probability at least $1-\delta$ , $\zeta$-\algoname: (a) terminates in no more than $T_{\text{sufficient}}$ rounds, and (b) returns a $\zeta$-best arm.
\end{theorem}

The pseudocode for the $\zeta$-\algoname{} is as below :
\vspace{-0.1em}
\begin{algorithm}[H]
   \caption{$\zeta$-\algoname}\label{alg:zeta_GRUB}
\begin{algorithmic}
   \STATE {\bfseries Input:} Regularization parameter $\rho$, Smoothness parameter $\epsilon$, Error bound $\delta$, Total arms $n$, Laplacian $L_G$, Sub-gaussianity parameter $\sigma$
   \STATE $t\leftarrow 0$\;
   \STATE $A = \{1,2,\dots, n\}$\;
   \STATE $t=0$\;
   \STATE $V_0 \leftarrow \rho L_G$\;
   \STATE $\mathcal{C}(G) \leftarrow $ {\tt Cluster-Identification}($L_G$)\;
   \FOR{$C\in \mathcal{C}(G)$}
   \STATE $t\leftarrow t+1$\;
   \STATE Pick random arm $k \in C$ to observe reward $r_{t,k}$\;
   \STATE $V_{t}\leftarrow V_{t-1}+\mathbf{e}_k\mathbf{e}^T_k$, and $\mathbf{x}_{t}\leftarrow \mathbf{x}_{t-1} + r_{t, k}\mathbf{e}_k$\;
   \ENDFOR
   \WHILE{$|A| > 1$}
   \STATE $t\leftarrow t+1$\;
    \STATE $\beta(t) \leftarrow 2\sigma\sqrt{14\log{\left(\frac{2n(t+1)^2}{\delta}\right)}} + \rho\epsilon$\;
   \STATE $k\leftarrow $ {\tt Sampling-Policy}($t, V_t, A$, $\mathcal{C}(G)$)\;
   \STATE Sample arm $k$ to observe reward $r_{t,k}$\;
   \STATE $V_{t}\leftarrow V_{t-1}+\mathbf{e}_k\mathbf{e}^T_k$\;
   \STATE $\mathbf{x}_{t}\leftarrow \mathbf{x}_{t-1} + r_{t,k}\mathbf{e}_k$\;
    \STATE $\hat{\pmb{\mu}}_{t}\leftarrow V_{t}^{-1}\mathbf{x}_{t}$\;
    \STATE $a_{\max} \leftarrow \underset{i\in A}{\arg\max} \left[\hat{\mu}^i_{t} - \beta(t_i)\sqrt{[V_{t}^{-1}]_{ii}}\right]$\;
    \STATE $A\leftarrow \left\{\mathbf{a}\in A~ | ~\hat{\mu}^{a_{\max}}_{t} - \hat{\mu}^a_{t} \leq  \beta(t_{a})\sqrt{[V_{t}^{-1}]_{aa}}\right.$\;
    \STATE $ \left. \qquad\qquad + \beta(t_{a_{\max}})\sqrt{[V_{t}^{-1}]_{a_{\max}a_{\max}}} \right\}$\;
    \STATE $A\leftarrow A/\left\{ a\in A ~|~ \beta(t_a)\sqrt{[V_t^{-1}]_{aa}} \leq \frac{\zeta}{2}\right\}$\;
   \ENDWHILE
   \RETURN $\arg\max \left\{\mu_i |~ i\in \{a \in [n] | \beta(t_a)\sqrt{[V_t^{-1}]_{aa}} \leq \frac{\zeta}{2}\} \cup A\right\}$
\end{algorithmic}
\end{algorithm}

 

\section{Experiments}\label{sec:exp}

For all our experiments, we use standard laptop with Intel® Core™ i7-10875H CPU @ 2.30GHz $\times$ 16 with 32 GB memory. We set the probability of error $\delta = 1e-3$, the penalizing constant $\rho=2.0$ and noise variance of the subgaussian distribution $ \sigma=2.0$. .For the additional graph information, we consider 2 cases: $G$ is a Stochastic Block model(SBM) with parameters $(p, q) = (0.9, 1e^{-4})$ and $G$ is a Barab{\'a}si–Albert(BA) graph with parameter $m=2$, both containing $10$ clusters. We record the stopping time for $20$ runs and plot the results. We evaluate \algoname{} with different sampling strategies from section~\ref{subsec:sampling_policies} and compare its performance to standard UCB algorithm~\citep{lattimore2020bandit}. The full code used for conducting experiments can be found at the following \color{blue}\href{https://github.com/parththaker/Bandits-GRUB}{Github repository}\color{black}.

Figure~\ref{fig:constant_cluster_all_algo} compares the baseline cyclic algorithm (UCB algorithm without graph information) with ~\algoname{} and its variants (\algoname{}-MVM, JVM-O, JVM-N) as listed in Section~\ref{subsec:sampling_policies}. The x-axis represents the number of arms while keeping the number of clusters constant. As can be seen, all the graph-based methods keep performing better compared to standard baseline UCB, which shows almost linear growth with the number of arms. Interestingly, note that JVM-O, JVM-N and MVM perform better with increase in the number of arms in the bandit problem. This is attributed to the fact that increasing number of arms while keeping the number of clusters static increases the \textit{density} of connections per arm and thereby improving performance.





\begin{figure}[H]
    \centering
    \begin{minipage}{0.45\textwidth}
        \centering
        \includegraphics[width=1.0\textwidth]{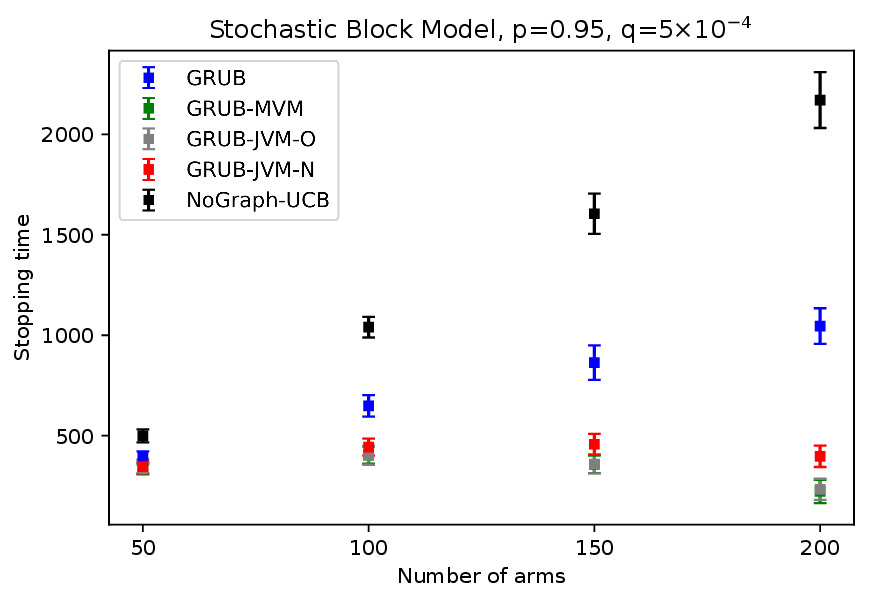} 
    \end{minipage}\hfill
    \begin{minipage}{0.45\textwidth}
        \centering
        \includegraphics[width=1.0\textwidth]{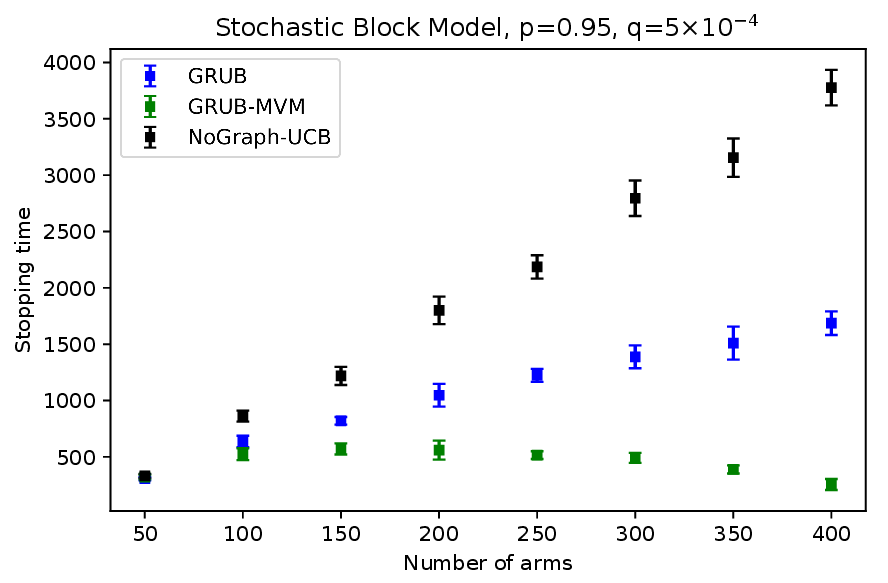}
    \end{minipage}
    \medskip

    \begin{minipage}{0.45\textwidth}
        \centering
       \includegraphics[width=1.0\textwidth]{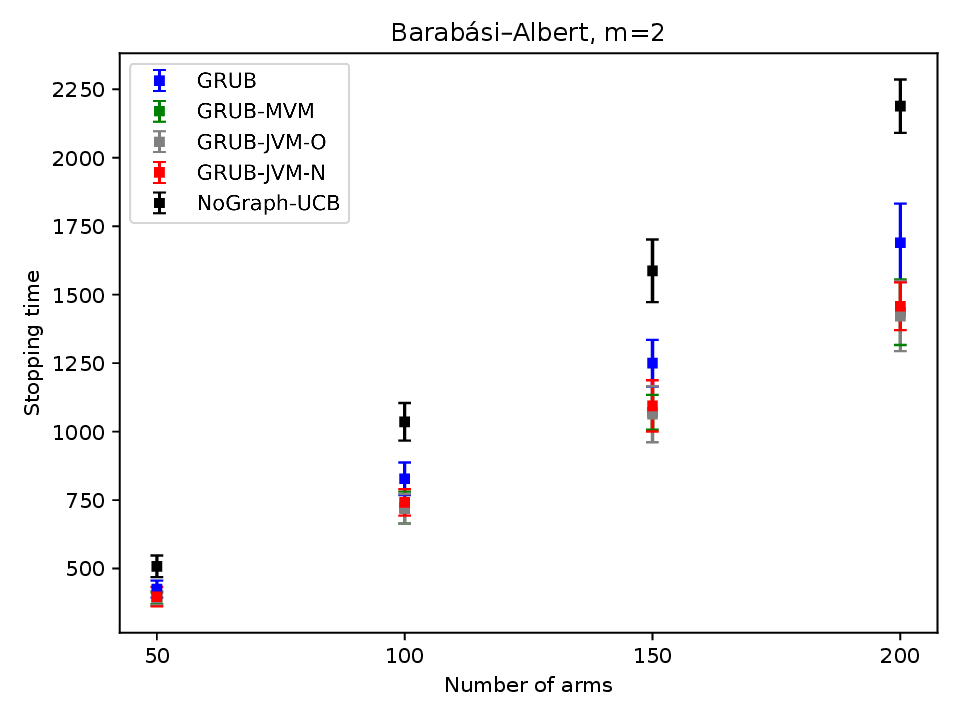}
    \end{minipage}\hfill
    \begin{minipage}{0.45\textwidth}
        \centering
        \includegraphics[width=1.0\textwidth]{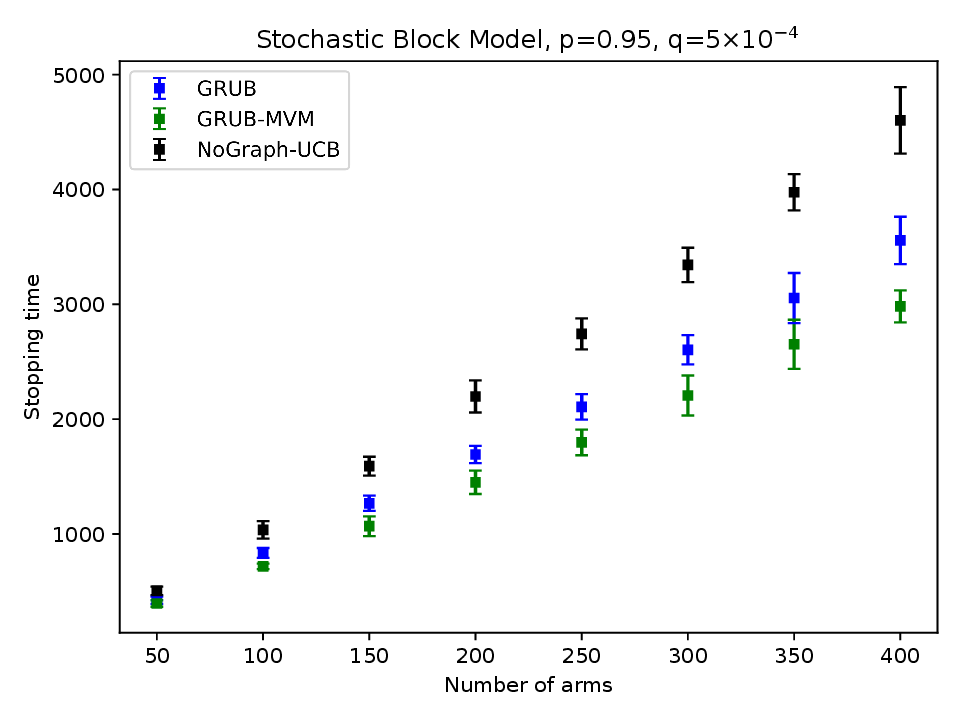} 
    \end{minipage}
    \caption{(Best seen in color) Stopping time vs number of arms of \algoname{} using various sampling protocols for SBM  ($(p, q) = (0.95, 1e-4)$) [Top] and BA ($m=2$) [Bottom] . Graph based pure exploration methods  outperforms the standard cyclic UCB method in terms of stopping time}
    \label{fig:constant_cluster_all_algo}
\end{figure}

\textbf{Real Dataset:} It is difficult to obtain a published dataset which exactly fits our problem of pure exploration with graph side information. In order to create a semi-real problem setup, we append an already existing network of users with a corresponding (synthetic) mean structure so as to satisfy the graph side information constraint. We use graphs from SNAP~\cite{snapnets} for these experiments. We sub-sample the graphs using Breadth-First Search (to retain connected components) to generate the graphs for our experiments. We use the LastFM~\cite{feather}, subsampled to 229 nodes and Github Social~\cite{rozemberczki2019multiscale} subsampled to 242 nodes. 
    
    \begin{figure}[H]
    \centering
    \begin{minipage}{0.45\textwidth}
        \centering
        \includegraphics[width=1.1\textwidth]{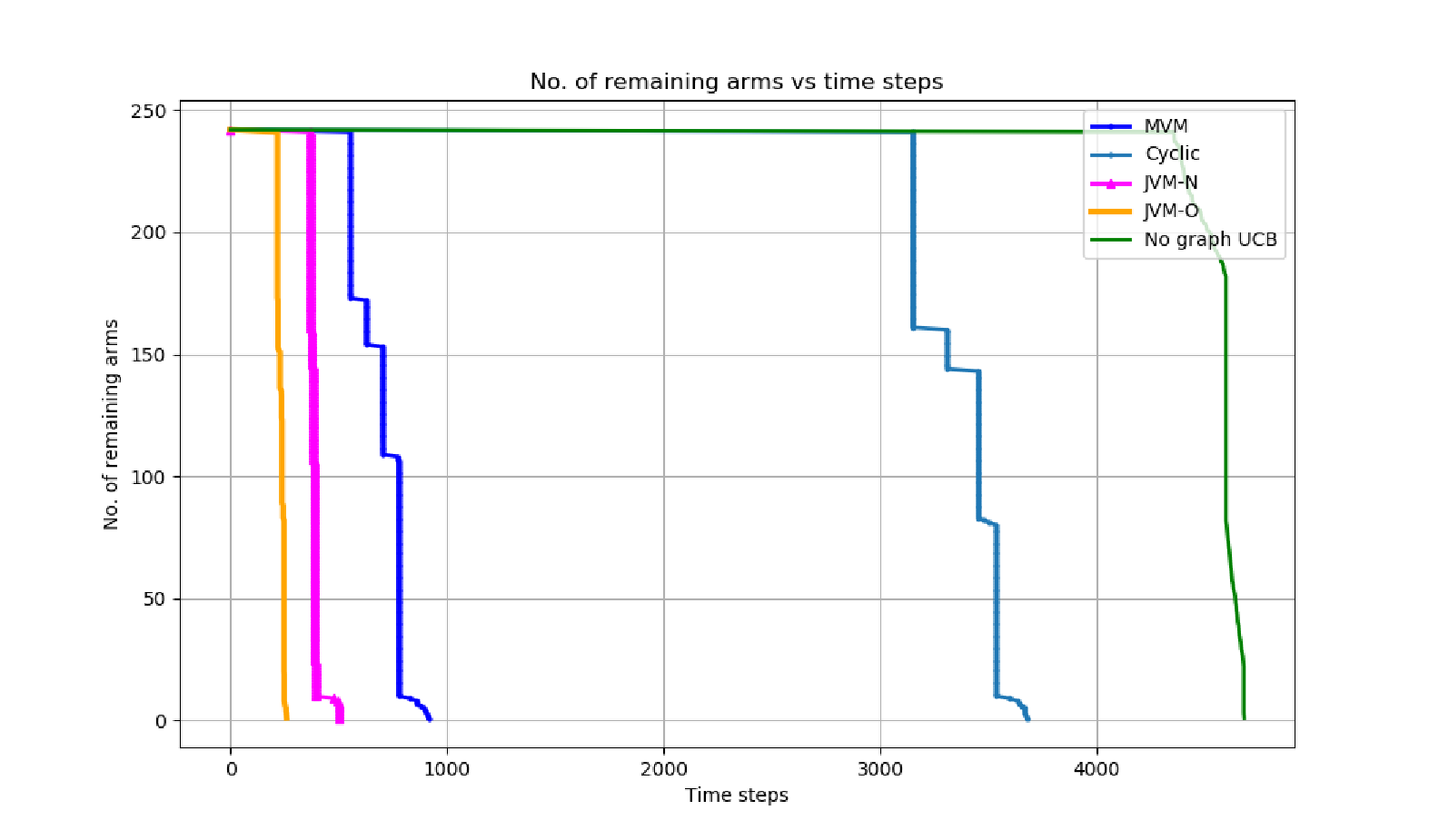} 
    \end{minipage}\hfill
    \begin{minipage}{0.45\textwidth}
        \centering
        \includegraphics[width=1.1\textwidth]{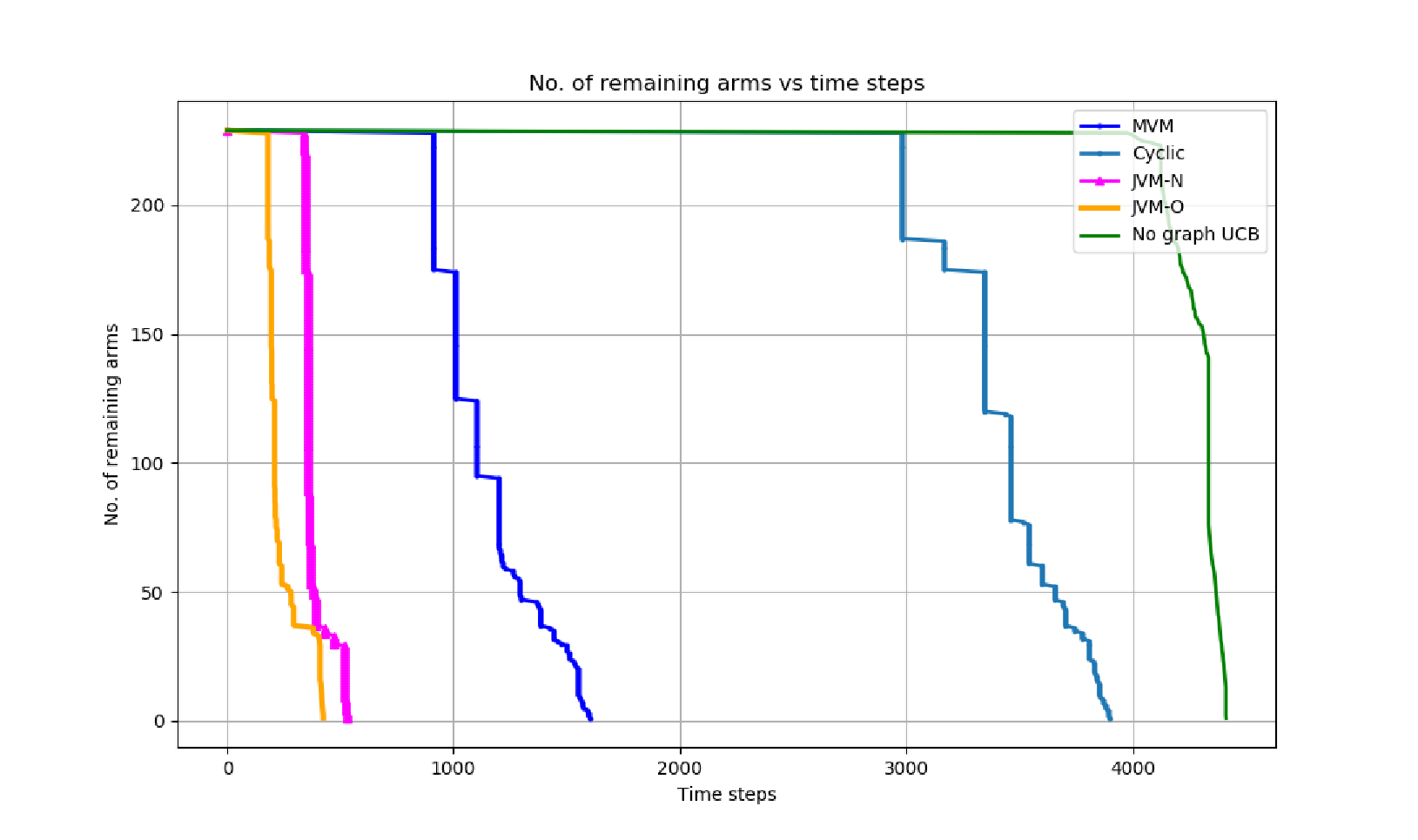} 
    \end{minipage}
    \caption{(Best seen in color) Cardinality of $|A_t|$ vs. time $t$ of \algoname{} using different sampling protocols for Github social graph (left) and LastFM graph (right). With no graph information, UCB requires orders of magnitude more samples compared to policies that use explicitly graph information. The cycic sampling policy is not as competitive on real world datasets}
    \label{fig:real}
\end{figure}

Figure~\ref{fig:real} plots the number of arms still in consideration $|A_t|$ vs. time $t$ for a single run of the pure exploration problems. This provides us better insights into the behaviour of \algoname{} with different sampling protocols (Section~\ref{subsec:sampling_policies}) and standard UCB approach. In all the experiments, it is evident that \algoname{} with any of the sampling policies outperform UCB algorithm~\cite{lattimore2020bandit}, which does not leverage the graph. Further within the various sampling policies, MVM sampling policy seems to outperform other sampling policies (Figure \ref{fig:real}). For both Github and LastFM datasets, the MVM policy obtains the best arm in $\sim 300$ rounds compared to traditional UCB that takes $\sim 4500$ rounds. A rigorous theoretical characterization of the above sampling policies is an exciting avenue for future research. 





\section{Discussion and Broader Impacts}\label{app:broader_impact}

In this work, we consider the problem of best arm identification (and approximate best arm identification) when one has access to information about the similarity between the arms in the form of a graph. We propose a novel algorithm \algoname{} for this important family of problems and establish sample complexity guarantees for the same. In particular, our theory explicitly demonstrated that benefit of this side information (in terms of the properties of the graph) in quickly locating the best or approximate best arms. We support these theoretical findings with experimental results in both simulated and real settings. 

{\bf Future Work and Limitations. }
We outline several sampling policies inspired by our theory in Section~\ref{sec:exp}; an extension of our theoretical results to account for these improved sampling policies is a natural candidate for further exploration. The algorithms and theory of this paper assume knowledge of (an upper bound) on the smoothness of the reward vector with respect to the graph. While this is where one uses domain expertise, this could be hard to estimate in certain real world problems. A generalization of the algorithmic and theoretical framework proposed here that is \emph{adaptive} to the unknown graph-smoothness is an exciting avenue for future work~\citep{cai2012adaptive, banerjee2020adaptive}. The sub-Gaussianity assumption of this work can also be generalized to other tail behaviors in follow up work. Another limitation of this work is that the statistical benefit of the graph-based quadratic penalization comes at a computational cost -- each mean estimation step involves the inversion of an $n\times n$ matrix which has a complexity of $O(n^2\log(n))$. However, an exciting recent line of work suggests that this matrix inversion can be made significantly faster when coupled with a spectral sparsification of the graph $G$~\citep{vishnoi2013lx, spielman2010spectral} while controlling the statistical impact of such a modification. In the context of this problem, this suggests a compelling avenue for future work that studies the statstics-vs-computation tradeoffs in using graph side information.  


For this work, we demonstrated the advantages of this side information in pure exploration problems, given knowledge of such an $\epsilon$. Extensions that consider goodness-of-fit and misspecification with respect to the graph $G$ and smoothness parameters $\epsilon$ are interesting avenues for follow up work. Finally, we focus on the ridge-type regularizer of the form $\langle\mathbf{\mu}, L_G\mathbf{\mu}\rangle$. For future work, it may be productive to expand to a much broader class of regularizers such as those of the form of $\|A\mathbf{\mu}\|^p_q$, where $A$ represents a information/ structural constraint matrix and $p,q$ are some positive  numbers. 

{\bf Potential Negative Social Impacts. }Our methods can be used for various applications such as drug discovery, advertising, and recommendation systems. In scientifically and medically critical applications, the design of the reward function becomes vital as this can have a significant impact on the output of the algorithm. One must take appropriate measures to ensure a fair and transparent outcome for various downstream stakeholders. With respect to applications in recommendation and targeted advertising systems, it is becoming increasingly evident that such systems may exacerbate polarization and the creation of filter-bubbles. Especially the techniques proposed in this paper could reinforce emerging polarization (which would correspond to more clustered graphs and therefore better recommendation performance) when used in such contexts. It will of course be of significant interest to mitigate such adverse outcomes by well-designed interventions  or by considering multiple similarity graphs that capture various dimensions of similarity. This is a compelling avenue for future work. 


\bibliographystyle{plain}
\bibliography{mybib}
\newpage


\newpage

\appendix

\section*{Appendix}
The appendix is organized as follows. Appendices~\ref{app:support}-\ref{app:eff_samp} and Appendix~\ref{app:support2} provide various supporting results and insights into our main theoretical results. Appendix~\ref{app:generic_sample_complexity} and Appendix~\ref{app:zeta_generic_sample_complexity} provide sample complexity guarantees for \algoname{} and $\zeta$-\algoname{} respectively. Appendix~\ref{app:lowerbounds} states and proves necessary conditions on the sample complexity, and Appendix~\ref{app:as_linear_bandit} presents a discussion on the incomparability of our graph bandits problem with that of linear bandits.

\section{Parameter estimation}\label{app:support}

At any time $T$, \algoname{}, along with the graph-side information, uses data gathered to estimate the mean $\hat{\pmb{\mu}}_T$ in order to decide the sampling and elimination protocols. The following lemma gives the estimation routine used for \algoname{}. 
\begin{lemma}\label{lem:mu_closed_form_proof}
The closed form expression of $\hat{\pmb{\mu}}_T$ is given by,
\begin{align}
    \hat{\pmb{\mu}}_T =  \left(\sum_{t=1}^T\mathbf{e}_{\pi_t}\mathbf{e}_{\pi_t}^T +\rho L_G\right)^{-1}\left(\sum_{t=1}^T \mathbf{e}_{\pi_t}r_t^{\pi_t}\right)
\end{align}
\end{lemma}
\begin{proof}

Using the reward data $\{r_{t,\pi_t}\}_{t=1}^T$ gathered up-to time $T$ and the sampling policy $\pmb{\pi}_T$, the mean vector estimate $\hat{\pmb{\mu}}_T$ is computed by solving the following laplacian-regularized least-square optimization schedule:
\begin{align}
    \hat{\pmb{\mu}}_T = \underset{\pmb{\mu}\in \mathbb{R}^{n}}{\mathop{\arg\min}}~~ \sum_{t=1}^{T}\left(\mu_{\pi_t} - r_{t, \pi_t}\right)^2 + \rho  \langle\pmb{\mu}, L_G\pmb{\mu}\rangle
\end{align}
where $\rho >0$ is a tunable penalty parameter. The above optimization problem can be equivalently written in the following quadratic form:
\begin{align}\label{eq:reform_mu_estimate}
    &\ \hat{\pmb{\mu}}_T = \underset{\pmb{\mu}\in \mathbb{R}^{n}}{\mathop{\arg\min}} \left(\langle\pmb{\mu}, V(\pmb{\pi}_T, G)\pmb{\mu}\rangle - 2\left\langle \pmb{\mu}, \left(\sum_{t=1}^T \mathbf{e}_{\pi_t}r_{t, \pi_t}\right)\right\rangle +\sum_{t=1}^T r_{t, \pi_t}^2\right) \nonumber
\end{align}
where $V(\pmb{\pi}_T, G)$ denotes,
\begin{align}
    V(\pmb{\pi}_T, G) = \sum_{t=1}^T\mathbf{e}_{\pi_t}\mathbf{e}_{\pi_t}^T +\rho L_G
\end{align}
In order to obtain $\hat{\pmb{\mu}}_T$, we compute vanishing point of the gradient as follows, 
\begin{align}
    &\ \left(\langle\pmb{\mu}, V(\pmb{\pi}_T, G)\pmb{\mu}\rangle - 2\left\langle \pmb{\mu}, \left(\sum_{t=1}^T \mathbf{e}_{\pi_t}r_{t, \pi_t}\right)\right\rangle +\sum_{t=1}^T r_{t, \pi_t}^2\right) |_{\pmb{\mu}= \hat{\pmb{\mu}}_T} = 0 ~~~~\nonumber\\ &\ \Rightarrow ~~~~
    \hat{\pmb{\mu}}_T =  V(\pmb{\pi}_T, G)^{-1}\left(\sum_{t=1}^T \mathbf{e}_{\pi_t}r_t^{\pi_t}\right)
\end{align}
\end{proof}

The sampling policy in \algoname{} uses the mean estimates and their high probability confidence bounds to eliminate suboptimal arm. In the following lemma we compute the high probability confidence bounds on the estimates of the mean and introduces the idea of effective samples of each arm given the graph side information.
\begin{lemma}\label{lem:variance_estimate_proof}
For any $T>k(G)$ and $i\in[n]$, the following holds with probability no less than $1-\frac{\delta}{w_i(\pmb{\pi}_T)}$:
\begin{align}\label{eq:hp_bounds_proof_appendix}
    |\hat{\mu}^{i}_T - \mu_{i}| \leq \sqrt{\frac{1}{t_{\text{eff}, i}}}\left(2\sigma\sqrt{14\log{\left(\frac{2w_i(\pmb{\pi}_T)}{\delta}\right)}} + \rho\|\pmb{\mu}\|_G\right)
\end{align}
where $w_i(\pmb{\pi}_T) = a_0nt_{\text{eff}, i}^2$ for some constant $a_0>0$, $\hat{\mu}^i_T$ is the $i$-th coordinate of the estimate from \ref{lem:mu_closed_form_proof} and,  $$t_{\text{eff}, i} =\frac{1}{ \left[\left(\sum_{t=1}^T\mathbf{e}_{\pi_t}\mathbf{e}_{\pi_t}^\top +\rho L_G\right)^{-1}\right]_{ii}}$$
\end{lemma}
\color{black}
\begin{proof}
Let the sequence of bounded variance noise and data gathered up-to time $T$ be denoted by  $\{\eta_t, r_{\pi_t, t}\}_{t=1}^T$. Let $S_T = \sum_{t=1}^T\eta_t\mathbf{e}_{\pi_t}$ and $N_T = \sum_{t=1}^T \mathbf{e}_{\pi_t}\mathbf{e}_{\pi_t}^T$. Using the closed form expression of $\hat{\pmb{\mu}}_T$ from eq.~\ref{lem:mu_closed_form_proof}, the difference between the estimate and true value $\hat{\mu}_T^i - \mu_i$ can be obtained as follows:
\begin{align}
    \hat{\mu}_T^i - \mu_i = \langle\mathbf{e}_i, \hat{\pmb{\mu}}_T - \pmb{\mu}\rangle = \langle \mathbf{e}_i, V_T^{-1}S_T - \rho V_T^{-1}L_G\pmb{\mu}\rangle\nonumber
\end{align}
The deviation $\hat{\mu}_T^i - \mu_i$ can be upper-bounded as follows:
\begin{align}
    |\langle \mathbf{e}_{i}, \hat{\pmb{\mu}}_T - \pmb{\mu}\rangle| \leq |\langle \mathbf{e}_{i}, V_T^{-1}S_T\rangle|+ |\langle \mathbf{e}_{i}, \rho V_T^{-1}L_G\pmb{\mu}\rangle|\nonumber
\end{align}
Further, in order to obtain the variance of the estimate $\hat{\pmb{\mu}}_T$, we bound the deviation $|\mu_T^{i} - \mu_i|$ by separately bounding $|\langle \mathbf{e}_{i}, V_T^{-1}S_T\rangle|$ and $|\langle \mathbf{e}_{i}\rho V_T^{-1}L_G\pmb{\mu}\rangle|$. 

With regards to the first term $\langle \mathbf{e}_{i}, V_T^{-1}S_T\rangle$, note that 
\begin{align}\label{eq:first_component_variance}
    \langle \mathbf{e}_i, V_T^{-1}S_T\rangle &\ = \left\langle \mathbf{e}_i,  V_T^{-1}\left(\sum_{t=1}^T \mathbf{e}_{\pi_t}\eta_t\right)\right\rangle \nonumber\\
    &\ = \sum_{t=1}^T\left\langle \mathbf{e}_i,  V_T^{-1} \mathbf{e}_{\pi_t}\right\rangle\eta_t\nonumber
\end{align}
Using a variant of Azuma's inequality~\cite{shamir2011variant, valko2014spectral}, for any $\kappa > 0$ the following inequality holds, 
\begin{align}
    \mathbb{P}\left(|\langle \mathbf{e}_i, V_{T}^{-1}S_{T}\rangle|^2 \leq \kappa^2\right) \geq 1-2\exp\left\{-\frac{\kappa^2}{56\sigma^2\sum_{t = 1}^{T}\left(\left\langle \mathbf{e}_i,  V_{T}^{-1} \mathbf{e}_{\pi_t}\right\rangle\right)^2}\right\}
\end{align}
Using the fact that $V_{T}  \succ \left(\sum_{t = 1}^{T}\mathbf{e}_{\pi_t}\mathbf{e}_{\pi_t}^T\right)$, we can further simplify the above bound using the following computation, 
\begin{align}
     \sum_{t = 1}^{T}\left(\left\langle \mathbf{e}_i,  V_{T}^{-1} \mathbf{e}_{\pi_t}\right\rangle\right)^2 &\ = \left\langle V_{T}^{-1}\mathbf{e}_i, \left(\sum_{t = 1}^{T}\mathbf{e}_{\pi_t}\mathbf{e}_{\pi_t}^T\right)  V_{T}^{-1} \mathbf{e}_i\right\rangle \nonumber\\ &\ \leq \langle\mathbf{e}_i, V_{T}^{-1}\mathbf{e}_i\rangle = [V_{T}^{-1}]_{ii}
\end{align}
Substituting $\delta'  = 2\exp\left\{-\frac{\kappa^2}{56\sigma^2\sum_{t = 1}^{T}\left(\left\langle \mathbf{e}_i,  V_{T}^{-1} \mathbf{e}_{\pi_t}\right\rangle\right)^2}\right\}$, we can finally conclude that given the historical data $\mathcal{F}_{T-1}$ till time $T-1$, following is true with probability $1-\delta'$,  
\begin{align}\label{eq:part_one_upper}
    |\langle \mathbf{e}_i, V_{T}^{-1}S_{T}\rangle|^2 \leq 56\sigma^2[V_T^{-1}]_{ii}\log{\left(\frac{2}{\delta'}\right)}
\end{align}
Second term $\langle \mathbf{e}_{i}, \rho V_T^{-1}L_G\pmb{\mu}\rangle$ can be upperbounded using cauchy-schwartz inequality,
\begin{align}\label{eq:part_two_upper}
    |\langle \mathbf{e}_{i}, \rho V_T^{-1}L_G\pmb{\mu}\rangle| &\ = \rho \langle \mathbf{e}_{i}, L_G\pmb{\mu}\rangle_{V_T^{-1}} \nonumber\\ &\ \leq \rho \sqrt{\langle \mathbf{e}_i ,V_T^{-1}\mathbf{e}_i\rangle}\sqrt{\langle L_G\pmb{\mu}, V_T^{-1}L_G\pmb{\mu}\rangle} \nonumber\\
    &\ \leq  \rho\sqrt{[V_T^{-1}]_{ii}}\|\pmb{\mu}\|_G
\end{align}
Combining the upperbound~\eqref{eq:part_two_upper}, ~\eqref{eq:part_one_upper} and substituting $\delta' = \frac{\delta}{w(\pmb{\pi}_T)}$ we get Lemma~\ref{lem:variance_estimate}. Hence proved.
\end{proof}

\section{Influence Factor}\label{app:res_dis_and_inf}

A key component in our characterization of the performance of \algoname{} is 
the \textit{influence factor} for each arm; recall that for a given graph $D$, $C_i(D)$ denotes the connected component that contains $i$. The influence factor for each arm is defined as, 

\begin{definition}
Let $D$ be a graph on the vertex set $[n]$. For each $j\in [n]$, define \textbf{influence factor $\mathfrak{I}(j, D)$} as:
\begin{align}\label{eq:d}
    \mathfrak{I}(j, D) = \begin{cases} \underset{i\in C_j(D), i\neq j}{\min}\{r_D(i,j)^{-1}\}&\ \text{ if }~~|C_j(D)|>1\\
    0 ~~~&\ \text{ otherwise }
    \end{cases}
\end{align}
where, $r_D(i,j)$ is the resistance distance between arm $i$ and $j$ on graph $D$ as in Definition~\ref{def:res_dis}. 
\end{definition}


Note that we refer the resistance distance without the parameter $\delta$, as the value of resistance distance is independent of the value of $\delta$. This happens due to the cancellation of $\delta$ factor in $R_{ii} + R_{jj}-R_{ji}-R_{ij}$. The influence factor can also be thought of as the minimum influence any arm $i$ in the connected component of arm $j$ has over the arm $j$




\color{black}
\section{Effective Samples}\label{app:eff_samp}

\begin{theorem}\label{thm:effective_samples_proof}
Let $\pmb{\pi}_T$ indicate the sampling policy until time $T$. Let $G$ be the given graph, $\mathfrak{I}(.,G)$ indicates the minimum influence factor for arms. Then effective samples can be lower bounded by,
\begin{align}
    t_{\text{eff}, i} \geq t_i + \frac{1}{2}\lfloor\min\{\rho\mathfrak{I}(i, G), \sum_{j\in C(i)}t_j\}\rfloor
\end{align}
where $t_i$ indicates the no. of samples of arm $i$ and $\lfloor~\cdot~\rfloor$ indicates the floor. 
\end{theorem}
\begin{proof}
Using Lemma~\ref{lem:v_t_upper_bound}, we have the following bound on $[V_T^{-1}]_{ii}$, 
\begin{align}
    &\ [V(\pmb{\pi}_T, G)^{-1}]_{ii} \leq \max\left\{ \frac{1}{t_i + \frac{\rho\mathfrak{I}(i, G)}{2}}, \frac{1}{t_i + \frac{t_C-t_i}{2}} \right\}
\end{align}
where $T$ is the total number of samples and $t_C$ is all the samples from the connected component $C(i)$ apart from arm $i$. Thus rewriting the equation for $t_{\text{eff}, i}$, we get,
\begin{align}
    t_{\text{eff}, i} \geq t_i + \frac{1}{2}\min\{\rho\mathfrak{I}(i, G), \sum_{j\in C(i)}t_j\}
\end{align}
Hence proved.
\end{proof}
\section{GRUB Sample complexity}\label{app:generic_sample_complexity}
In order to compute the sample complexity for \algoname{}, we classify the arms into two categories: competitive and non-competitive. The split of arms into these two categories is not required for the algorithm, but provides tighter complexity bounds as will be observed in this appendix. The division of the arms is contingent on its suboptimality and the structure of the provided graph side information. A modified version of the Definition~\eqref{def:b_d_set_def} of competitive set and non-competitive set is as follows:

\begin{definition}\label{def:b_d_set_def_full_gory}
Fix $\pmb{\mu}\in \mathbb{R}^n$, graph $D$, regularization parameter $\rho$, confidence parameter $\delta$, and smoothness parameter $\epsilon$ and noise variance $\sigma$. We define $\mathcal{H}$ to be the set of competitive arms and $\mathcal{N}$ to be the set of non-competitive arms as follows: 
\small
\begin{align}\label{def:b_i_d_i_full_gory}
     \mathcal{H}(D, \pmb{\mu},\delta, \rho, \epsilon) &= \left\{j\in [n] \big| \Delta_i \leq 2\sqrt{\frac{2}{\rho\mathfrak{I}(i)}}\left(2\sigma\sqrt{14\log{\left(\frac{2a_0n\rho^2\mathfrak{I}(i)^2}{\delta}\right)}} + \rho\epsilon\right)\right\}\nonumber,\\
     \mathcal{N}(D, \pmb{\mu}, \delta, \rho, \epsilon) &\triangleq [n]\setminus \mathcal{H}(D, \pmb{\mu},\delta, \rho, \epsilon) \nonumber
\end{align}
\normalsize
\end{definition}
When the context is clear, we will use suppress the dependence on the parameters in Definition~\ref{def:b_d_set_def_full_gory}.

Further, we derive an expression for the worst-case sample complexity by analysing the number of samples required to eliminate arms with different difficulty levels, i.e. arms in competitive set and non-competitive set. We first derive the sample complexity results for the case when graph $G$ is connected and then extend it to disconnected graphs. 

\begin{lemma}\label{lem:connected_sample_complexity}
Consider $n$-armed bandit problem with mean vector $\pmb{\mu}\in \mathbb{R}^n$. Let $G$ be a given connected similarity graph on the vertex set $[n]$, and further suppose that $\pmb{\mu}$ is $\epsilon$-smooth. Define 
\begin{align}
    T_{\text{sufficient}} \triangleq \sum_{i\in \mathcal{H}} \frac{1}{\Delta_i^2}\left[c_1\log{\frac{c_2}{\delta\Delta_i}} + \frac{\rho\epsilon}{2}\right] + \max_{i\in \mathcal{N}}\left\{\frac{2}{\Delta_i^2}\left[c_1\log{\frac{c_2}{\delta\Delta_i}} + \frac{\rho\epsilon}{2}\right]\right\}
\end{align}
Then, with probability at least $1-\delta$, \algoname{}: (a) terminates in no more than $T_{\text{sufficient}}$ rounds, and (b) returns the best arm $a^\ast = \arg\max_i \mu_i$. 
\end{lemma}


\begin{proof}
With out loss of generality, assume that $a^* = 1$. Let $\{t_i\}_{i=1}^n$ denote the number of plays of each arm upto time $T$. By Lemma~\ref{lem:variance_estimate}, we can state that,
\begin{align}
    \mathbb{P}\left(|\hat{\mu}^{i}_T - \mu_{i}| \geq \gamma_i(\pmb{\pi}_T)\right) \leq \frac{2\delta}{a_0nt_{\text{eff}, i}^2}
\end{align}
where, $\gamma_i(\pmb{\pi}_T) = \beta_i(\pmb{\pi}_T)\sqrt{t_{\text{eff}, i}^{-1}}$ and  $\beta_i(\pmb{\pi}_T) = \left(2\sigma\sqrt{14\log{\left(\frac{2a_0nt_{\text{eff}, i}^2}{\delta}\right)}} + \rho\|\pmb{\mu}\|_G\right).$ 

As is reflected in the elimination policy~\eqref{eq:elimination_routine}, at any time $t$, arm 1 can be mistakenly eliminated in \algoname{} only if $\hat{\mu}_t^{i} > \hat{\mu}_t^1 + \gamma_i(\pmb{\pi}_t) +\gamma_1(\pmb{\pi}_t)$. Let $T_s$ be the stopping time of \algoname{}, then the total failure probability for \algoname{} can be upper-bounded as,
\begin{align}
    \mathbb{P}(\text{Failure}) &\ \leq \sum_{t=2}^{T_s}\sum_{i=2}^n \mathbb{P}\left(\hat{\mu}_t^i \geq \hat{\mu}^1_t + \gamma_i(\pmb{\pi}_t) +\gamma_1(\pmb{\pi}_t)\right)\nonumber
\end{align}
Note that $\mathbb{P}\left(\hat{\mu}_t^i \geq \hat{\mu}^1_t + \gamma_i(\pmb{\pi}_t) +\gamma_1(\pmb{\pi}_t)\right) \leq \left[\mathbb{P}\left(\hat{\mu}_t^i \geq \mu^i + \gamma_i(\pmb{\pi}_t)\right) + \mathbb{P}\left(\hat{\mu}_t^1 \leq \mu^1 - \gamma_1(\pmb{\pi}_t)\right) \right]$, provided that $\gamma_i(\pmb{\pi}_t), \gamma_1(\pmb{\pi}_t)\leq \frac{\Delta_i}{2}$. Hence the failure probability can be upperbounded as,
\begin{align}
    \mathbb{P}(\text{Failure}) &\ \leq  \sum_{i=2}^n \sum_{t=2}^{T_s} \left[\mathbb{P}\left(\hat{\mu}_t^i \geq \mu^i + \gamma_i(\pmb{\pi}_t)\right) + \mathbb{P}\left(\hat{\mu}_t^1 \leq \mu^1 - \gamma_1(\pmb{\pi}_t)\right) \right]
\end{align}
conditioned on $\gamma_i(\pmb{\pi}_T), \gamma_1(\pmb{\pi}_T)\leq \frac{\Delta_i}{2}$. 

Let $a_0 \geq 4\sum_{t=1}^\infty t_{\text{eff}, i}^{-2}$, then from Lemma~\ref{lem:variance_estimate}, 
\begin{align}
    \mathbb{P}(\text{Failure}) &\ \leq \sum_{i=2}^n \sum_{t=2}^{T_s} \frac{2\delta}{a_0nt_{\text{eff}, i}^2}\nonumber\\
    &\ \leq \delta
\end{align}
The finiteness of the infinite sum of ${t_{\text{eff}, i}}^{-2}$ can be found in Lemma~\ref{lem:infinite_sum_bounded}. 

Thus, in order to keep $\mathbb{P}(\text{Failure}) \leq \delta$, it is sufficient if, at the time of elimination of arm $i$, we have enough samples to ensure,
\begin{align}\label{eq:elimination_criteria_1}
   \gamma_i(\pmb{\pi}_T) &\ \leq \frac{\Delta_i}{2}\nonumber\\
   \sqrt{\frac{1}{t_{\text{eff}, i}}}\left(2\sigma\sqrt{14\log{\left(\frac{2a_0nt_{\text{eff}, i}^2}{\delta}\right)}} + \rho\epsilon\right) &\ \leq \frac{\Delta_i}{2} 
\end{align}
In the absence of graph information, equation~\eqref{eq:elimination_criteria_1} devolves to the same sufficiency condition for number of samples required for suboptimal arm elimination as~\cite{even2006action}, upto constant factor. Rewriting the above equation, 
\begin{align}\label{eq:suff_sample_a_i_1}
    \frac{\log{\left(a_i\right)}}{a_i} &\ \leq \sqrt{\frac{\delta}{d_1}}\frac{\Delta_i^2}{d_0}
    \end{align}
where $d_0 = 64\times 14\sigma^2, d_1 = 2n a_0 e^{\frac{\rho^2\epsilon^2}{4\times14\sigma^2}}$ and $a_i = \sqrt{\frac{d_1}{\delta}}t_{\text{eff}, i}$. The following bound on $a_i$ is sufficient to satisfy eq.~\eqref{eq:suff_sample_a_i_1}, 
    \begin{align}
    a_i &\ \geq 2\sqrt{\frac{d_1}{\delta}}\frac{d_0}{\Delta_i^2}\log{\left(\sqrt{\frac{d_1}{\delta}}\frac{d_0}{\Delta_i^2}\right)}\nonumber
\end{align}
Resubstituting $t_{\text{eff}, i}$, we obtain the sufficient number of plays required to eliminate arm $i$ as, 
\begin{align}
    t_{\text{eff}, i} &\ \geq \frac{c_1}{\Delta_i^2}\left[\log{\left(\frac{c_2}{\delta^{\frac{1}{2}}\Delta_i^2}\right)} + c_3\right]
\end{align}
where $c_1 = 2 \times 64 \times 14\sigma^2$, $c_2 = 64 \times 14\sigma^2\sqrt{2na_0}$ and $c_3 = \frac{\rho^2\epsilon^2}{8 \times 14\sigma^2}$. In the further text we are suppressing the powers of $\delta, \Delta_i$ within the log factor as it adds only a constant multiple to the lower bound. 

The further part of the proof we use the following bound on $t_{\text{eff}, \cdot}$ from Theorem~\ref{thm:effective_samples_proof} as follows:
\begin{align}
    t_{\text{eff}, i} \geq t_i +\frac{1}{2}\min\left\{\rho\mathfrak{I}(i), T-t_{i}\right\}~~~\forall i\in [n]
\end{align}
Hence a sufficiency condition for the~\algoname{} to produce the best-arm with probability $1-\delta$ is given when both the following conditions are satisfied,
\begin{align}
    t_i + \frac{\rho\mathfrak{I}(i)}{2} \geq \frac{1}{\Delta_i^2}\left[c_1\log\left(\frac{c_2}{\delta\Delta_i}\right) + \frac{\rho\epsilon}{2}\right]
\end{align}
and, 
\begin{align}
    T+t_i \geq T \geq  \frac{2}{\Delta_i^2}\left[c_1\log\left(\frac{c_2}{\delta\Delta_i}\right) + \frac{\rho\epsilon}{2}\right]
\end{align}

From the Definition~\ref{def:b_d_set_def_full_gory} of competitive arms $\mathcal{H}$ and non-competitive arms $\mathcal{N}$, we have,  
\begin{align}
    \mathcal{H} = \left\{j\in [n] \big| \Delta_i \leq 2\sqrt{\frac{2}{\rho\mathfrak{I}(i)}} \left(2\sigma\sqrt{14\log{\left(\frac{2a_0n\rho^2\mathfrak{I}(i)^2}{\delta}\right)}} + \rho\epsilon\right)\right\}
\end{align}
After the first $\max_{i\in \mathcal{N}}\left\{\frac{2}{\Delta_i^2}\left[c_1\log{\frac{c_2}{\delta\Delta_i}} + \frac{\rho\epsilon}{2}\right]\right\}$ samples, all arms in $\mathcal{N}$ are eliminated. Further, let $k_1$ be the index of the first arm to be eliminated (in $\mathcal{H}$) and $t^*_{k_1}$ be the number of samples of arm $k_1$ before getting eliminated then the total number of additional time steps played until the arm $k_1$ is eliminated is at most $|\mathcal{H}|t^*_{k_1}$. Let $k_2$ be the index of the next arm in $\mathcal{H}$ to be eliminated. The number of additional plays until the next arm is eliminated is given by $(|\mathcal{H}|-1)[t^*_{k_2} - t^*_{k_1}]$ and so on. 

Summing up all the samples required to converge to the optimal arm is given by, (let $t^*_{k_0} = 0$)
\begin{align}
    \sum_{h=1}^{|\mathcal{H}|}(|\mathcal{H}| - h))[t^*_{k_h} -t^*_{k_{h-1}}] = \sum_{h=1}^{|\mathcal{H}|-1}t^*_{k_h} = \sum_{i\in \mathcal{H}/1} t^*_i
\end{align}

Hence the final sample complexity can be computed as follows:
\begin{itemize}
    \item Number of plays required for arms in $\mathcal{H}$ :
    \begin{align}
        \sum_{i\in \mathcal{H}/1}t^*_i \geq \sum_{i\in \mathcal{H}/1}\frac{1}{\Delta_i^2}\left[c_1\log{\frac{c_2}{\delta\Delta_i}} + \frac{\rho\epsilon}{2}\right]
    \end{align}
    \item Number of plays required for all the arms in $\mathcal{N} := [n]/\mathcal{H}$ to be eliminated:
    \begin{align}
        T \geq \max_{i\in\mathcal{N}}\left\{\frac{2}{\Delta_i^2}\left[c_1\log{\frac{c_2}{\delta\Delta_i}} + \frac{\rho\epsilon}{2}\right]\right\}
    \end{align}
\end{itemize}

Hence the final sample complexity can be given by,
\begin{align}
    T_{\text{sufficient}}  \triangleq \max_{i\in \mathcal{N}}\left\{\frac{2}{\Delta_i^2}\left[c_1\log{\frac{c_2}{\delta\Delta_i}} + \frac{\rho\epsilon}{2}\right]\right\}  + \sum_{i\in \mathcal{H}/1} \frac{1}{\Delta_i^2}\left[c_1\log{\frac{c_2}{\delta\Delta_i}} + \frac{\rho\epsilon}{2}\right]
\end{align}

Hence proved.
\end{proof}

We extend Lemma~\ref{lem:connected_sample_complexity} to the case when graph $G$ has disconnected clusters.

\textbf{Note:} The following theorem stated in the main paper has a typographical error in the equation for $T_\text{sufficient}$ in place of $\mathop{\arg\min}$ it is supposed to be $\min$.

\begin{theorem}\label{thm: sample_complexity-proof}
Consider $n$-armed bandit problem with mean vector $\pmb{\mu}\in \mathbb{R}^n$. Let $\mathcal{G}$ be the set of subgraphs of given similarity graph $G$ on the vertex set $[n]$, and further suppose that $\pmb{\mu}$ is $\epsilon$-smooth. Define 
\begin{align}\label{eq:T_sample_generic_base}
    T_{\text{sufficient}}  \triangleq \mathop{\min}_{D\in \mathcal{G}} \sum_{C\in\mathcal{C}_D}\left[\sum_{i\in C\cap\mathcal{H}_D} \frac{1}{\Delta_i^2}\left[c_1\log{\frac{c_2}{\delta\Delta_i}} + \frac{\rho\epsilon}{2}\right]  + \max_{i\in C\cap\mathcal{N}_D}\left\{\frac{2}{\Delta_i^2}\left[c_1\log{\frac{c_2}{\delta\Delta_i}} + \frac{\rho\epsilon}{2}\right]\right\}\right]
\end{align}
where $\Delta_i = \mu^*-\mu_i$ for all suboptimal arms, $\mathcal{H}_D$ and $\mathcal{N}_D$ are as in Definition~\ref{def:b_d_set_def_full_gory}, $\mathcal{C}_D$ is the set of connected components of a subgraph $D\in\mathcal{G}$ and $c_1, c_2$ are constants independent of system parameters. Then, with probability at least $1-\delta$, \algoname{}: (a) terminates in no more than $T_{\text{sufficient}}$ rounds, and (b) returns the best arm $a^\ast = \mathop{\arg\max}_i \mu_i$. 
\end{theorem}
\begin{proof}
Let $\mathcal{C}_G$ denote the connected components of graph $G$. From Lemma~\ref{lem:connected_sample_complexity}, the number of samples for each connected component $C\in \mathcal{C}_G$ can be given as,
\begin{align}\label{eq:intermediate_GRUB}
    T_{\text{sufficient}} = \left[\sum_{i\in C\cap\mathcal{H}} \frac{1}{\Delta_i^2}\left[c_1\log{\frac{c_2}{\delta\Delta_i}} + \frac{\rho\epsilon}{2}\right]+ \max_{i\in C\cap\mathcal{N}}\left\{\frac{2}{\Delta_i^2}\left[c_1\log{\frac{c_2}{\delta\Delta_i}} + \frac{\rho\epsilon}{2}\right]\right\}\right]
\end{align}
We can obtain the sample complexity for obtaining the best arm by summing it over all the components $C\in \mathcal{C}$, gives us the sample complexity for ~\algoname{} while considering graph $G$. 
\begin{align}\label{eq:intermediate_full_GRUB}
    T_{\text{sufficient}} = \sum_{C\in \mathcal{C}_G}\left[\sum_{i\in C\cap\mathcal{H}} \frac{1}{\Delta_i^2}\left[c_1\log{\frac{c_2}{\delta\Delta_i}} + \frac{\rho\epsilon}{2}\right]+ \max_{i\in C\cap\mathcal{N}}\left\{\frac{2}{\Delta_i^2}\left[c_1\log{\frac{c_2}{\delta\Delta_i}} + \frac{\rho\epsilon}{2}\right]\right\}\right]
\end{align}

Any subgraph $D$ of graph $G$ satisfies,
\begin{align}
\langle \pmb{\mu}, L_G\pmb{\mu}\rangle \leq \epsilon \Rightarrow \langle \pmb{\mu}, L_D\pmb{\mu}\rangle \leq \epsilon
\end{align}


As seen in Definition~\ref{def:b_d_set_def_full_gory}, the influence factor is instrumental in deciding the competitive and non-competitive sets, which further dictates the sample complexity bounds. Further, notice from Lemma~\ref{lem:d_change_subgraph} that the influence factor $\mathfrak{I}(i, D)$ is not monotonic when considering subgraph $D$ of graph $G$.  Hence considering a subgraph of $G$ could potentially increase the number of non-competitive arms and provide us with a tighter bound on the performance for \algoname{}. 

Hence $T_{\text{sufficient}}$ in ~\eqref{eq:intermediate_GRUB} can be made tighter by considering the minimum value over the entire set of subgraphs $\mathcal{G}$. 
\end{proof}

We next derive sample complexity upper bounds for \algoname{} in certain illuminating special cases. 



\begin{corollary}[Isolated clusters]\label{cor:sample_complexity_clusters}
Consider the setup as in Theorem~\ref{thm: sample_complexity} with the further restriction that $G$ consists of a subgraph $F$ such that optimal node is isolated and arms $[2, \dots, n]$ are split in $k$ clusters and $\Delta_i \geq 2\sqrt{\frac{2}{\rho\mathfrak{I}(i,F)}}\left(2\sigma\sqrt{14\log{\left(\frac{2a_0n\rho^2\mathfrak{I}(i, F)^2}{\delta}\right)}} + \rho\epsilon\right)$, $\forall i\in [2, \dots, n]$. Define
\begin{align}
    T_{\text{sufficient}} \triangleq \sum_{C\in \mathcal{C}_F/1} \max_{j\in C} \frac{2}{\Delta_j^2}\left[c_1\log{\left(\frac{c_2}{\delta\Delta_i}\right) + \frac{\rho\epsilon}{2}}\right]
\end{align}
Then, with probability at least $1-\delta$, \algoname{}: (a) terminates in no more than $T_{\text{sufficient}}$ rounds, and (b) returns the best arm $a^\ast = \arg\max_i \mu_i$.
\end{corollary}
Corollary~\ref{cor:sample_complexity_clusters}  shows that in scenarios where the arms are well clustered, the sample complexity of \algoname{} can scale with the number of clusters, a quantity that is typically significantly smaller than the total number of nodes in the graph. 


\begin{corollary}[Star graph]\label{cor:sample_complexity_star}
Consider the setup as in Theorem~\ref{thm: sample_complexity} with the further restriction that $G$ consists of a star subgraph with the central node as the optimal arm and $\Delta_i \leq 2\sqrt{\frac{2}{\rho\mathfrak{I}(i,F)}}\left(2\sigma\sqrt{14\log{\left(\frac{2a_0n\rho^2\mathfrak{I}(i, F)^2}{\delta}\right)}} + \rho\epsilon\right)$, $\forall i\in [2, \dots, n]$. Define
\begin{align}
    T_{\text{sufficient}} \triangleq \sum_{i=2}^n \frac{1}{\Delta_i^2}\left[c_1\log\left(\frac{c_2}{\delta\Delta_i}\right)+\frac{\rho\epsilon}{2}\right]
\end{align}
Then, with probability at least $1-\delta$, \algoname{}: (a) terminates in no more than $T_{\text{sufficient}}$ rounds, and (b) returns the best arm $a^\ast = \arg\max_i \mu_i$.
\end{corollary}

In Corollary~\ref{cor:sample_complexity_star},  $T_{\text{sufficient}}$ is the same sample complexity as vanilla best arm identification, upto constant factors which is due to the fact that pulling one of the spoke arms does not yield much information about the other spoke arms, and this is the exact situation in the standard pure exploration setting.

\section{Lower bounds}{\label{app:lowerbounds}}
In this section we give a lower bound on the sample complexity for any $\delta$-PAC to return the best arm for a $n$ armed bandit problem along with graph side information.
\begin{theorem}\label{thm:lower-bound-known-proof}
Given an $n$-armed bandit model with associated mean vector $\pmb{\mu}\in \mathbb{R}^n$ and similarity graph $G$ smooth on $\pmb{\mu}$, i.e. $\langle\pmb{\mu}, L_G\pmb{\mu}\rangle \leq \epsilon$, for any $0<\epsilon <\epsilon_0$. Let $G = ([n],E)$ be the graph with only $k$ isolated cliques and w.l.o.g let arm 1 be the optimal arm. Then define
\begin{equation}
    T_{\text{necessary}} = \sum_{C\in \mathcal{C}_G/{C^*}}\min_{j\in C} \left\{\frac{4\sigma^2\log 5}{(\Delta_j - \sqrt{\epsilon})^2}\right\} + \sum_{j\in C^*/1 } \frac{4\sigma^2\log 5}{\Delta_j ^2}
\end{equation}
where $C^*$ is the clique with the optimal arm and $\epsilon_0 := \underset{i\in [n]/1, j\in C(i)}{\min}\left[\Delta_j \left[1 - \frac{\Delta_i}{\sqrt{\Delta_i^2 + \Delta_j^2}}\right]\right]^2$. Then any $\delta$-PAC algorithm will need at-least $T_{\text{necessary}}$ steps to terminate, provided $\delta \leq 0.1$. 
\end{theorem}
\begin{proof}

We prove the theorem in two steps: Firstly, we construct the sample complexity lower bound for the similarity graph with the isolated optimal arm and a clique of rest of the sub-optimal arms, followed by step 2 the sample complexity lower bound for a graph with single cluster

\textbf{Step 1:}

Consider a $n+1$ armed bandit problem with mean vector $\pmb{\mu}\in \mathbb{R}^{n+1}$ and similarity graph $M$ with an isolated optimal arm (arm 1) and $n$-clique cluster of suboptimal arms, satisfying the condition for smoothness of rewards over the graph,i.e., $\langle\pmb{\mu}, L_M\pmb{\mu}\rangle\leq \epsilon$. Then the following holds
\begin{align}\label{eq:smoothness_criterion_clique}
    \underset{i\neq 1}\max~\mu_i \leq \underset{j\neq 1}\min~\{\mu_j + \sqrt{\epsilon}\}
\end{align}

Assume that ordering of mean in $n$-clique of suboptimal arms is known. From ~\citep{kaufmann2015complexity}, there exists a $\delta$-PAC algorithm, for $\delta\leq 0.1$, which can successful identify the best arm for the subproblem with just the optimal arm and arm with the maximum mean in the $n$-clique cluster, i.e. $j' = \underset{j\neq 1}\arg\max \mu_j$ with the total number of samples given by,
\begin{align}
    T \geq \frac{4\log 5\sigma^2}{\Delta^2_{j'}}
\end{align}


Now consider the case where the ordering of the mean in $n$-clique is unknown. In order to remove all the suboptimal arms provided $\epsilon \leq \min_{j\neq 1}~\Delta_j^2$ and \eqref{eq:smoothness_criterion_clique} holds, it is suffices to be able to distinguish between the optimal arm and a hypothetical suboptimal arm with mean $\mu_j+\sqrt{\epsilon}$ where $j$ is any arm from suboptimal $n$-clique, and the minimum number of samples required by any $\delta$-PAC algorithm to successfully identify the best arm with $\delta \leq 0.1$ is given by,
\begin{align}
    T \geq \frac{4\log 5\sigma^2}{(\Delta_j - \sqrt{\epsilon})^2}
\end{align}
The best performance in terms of sample complexity out of all the random choice of arm from the suboptimal $n$-clique cluster is, 
\begin{align}
    T \geq \underset{j\neq 1}{\min}\left\{\frac{4\log 5\sigma^2}{(\Delta_j - \sqrt{\epsilon})^2}\right\}
\end{align}

Given $\epsilon_0 := \underset{i\in [n]/1, j\in C(i)}{\min}\left[\Delta_j \left[1 - \frac{\Delta_i}{\sqrt{\Delta_i^2 + \Delta_j^2}}\right]\right]^2$ and $\epsilon < \epsilon_0$, it can be verified that for any arm $i,j \neq 1$,
\begin{align}
    \underset{j\neq 1}{\min}\frac{4\log 5\sigma^2}{(\Delta_j - \sqrt{\epsilon})^2} < \frac{4\log 5\sigma^2}{\Delta_i^2} + \frac{4\log 5\sigma^2}{\Delta_j^2}
\end{align}
where the left hand side corresponds to the sample complexity lower bound of removing the suboptimal arms $i,j$ with the graph side information and the right hand side corresponds to the same without graph side information.

Hence it can be inferred that it is inefficient to remove the arms individually (disregarding the graph information).

\textbf{Step 2 :}

Consider a $n+1$ armed bandit problem with mean vector $\pmb{\mu}\in \mathbb{R}^{n+1}$ with a given similarity graph $N$ such that $\langle \pmb{\mu}, L_N\pmb{\mu}\rangle\leq \epsilon$. Let all the suboptimal arms be connected to the optimal arm. 

Here we show by an adversarial example that it is not possible to have a lower bound on the sample complexity which scales better than,
\begin{align}
    T \geq \sum_{j\neq 1}\frac{4\log 5\sigma^2}{\Delta_j^2}
\end{align}

There exists a $\delta$-PAC algorithm which can determine that arms $j=3, \dots, n$ are suboptimal after $T \geq \sum_{j\neq 1, 2}\frac{1}{\Delta_j^2}$ samples. From the smoothness of rewards on the similarity graph $N$ we know that,
\begin{align}
    - \sqrt{\epsilon} \leq \mu_1 - \mu_j  \leq \sqrt{\epsilon}~~~\forall j \in [2, 3, \dots, n] 
\end{align}
This information does not help us identify or even reduce the number of samples required to identify optimal arm between arm 1 and arm 2. Thus no $\delta$-PAC algorithm, $\delta \leq 0.1$, can determine the optimal arm from arm $1$ and arm $2$ without an additional $\frac{4\log 5\sigma^2}{\Delta_2^2}$ samples for determining the best arm. 

Using above two steps, we construct the proof for lower bound as follows: 

Now consider the graph side information as defined in the theorem, and let $\mathcal{C}_G$ denote the set of connected components of graph $G$ and $C^*\in \mathcal{C}_G$ be the component containing the optimal arm. Finding the best arm in this setup requires elimination of the suboptimal arms with in the connected component containing optimal arm $j \in C^*$ and elimination of the other connected components with suboptimal arms $j \in \mathcal{C}_G/C^*$. Hence, the sample complexity lower bounds~\citep{kaufmann2015complexity, kaufmann2016complexity} for any $\delta$-PAC algorithm with $\delta \leq 0.1$ to eliminate these arms using the tools developed in step 1 and step 2, is given by
\begin{align}
    T &\ \geq \sum_{j\in C^*/1} \frac{4\sigma^2\log 5}{\Delta_j^2} + \sum_{C\in \mathcal{C}_G/{C^*}} \min_{j\in C} \left\{\frac{4\sigma^2\log 5}{(\Delta_j - \sqrt{\epsilon})^2}\right\}
\end{align}
\end{proof}

\section{$\zeta$-GRUB Sample complexity proof}\label{app:zeta_generic_sample_complexity}
\begin{definition}\label{def:b_d_set_def_full_gory_zeta}
Fix $\pmb{\mu}\in \mathbb{R}^n$, graph $D$, confidence parameter $\delta$, noise variance $\sigma$, and relaxation parameter $\zeta$. We define $\mathcal{H}$ to be the set of competitive arms and $\mathcal{N}$ to be the set of non-competitive arms for $\zeta$-GRUB as follows: 
\small
\begin{align}\label{def:b_i_d_i_full_gory_zeta}
     \mathcal{H}(D,\pmb{\mu},\delta,\zeta) &= \left\{j\in [n] \big| \Delta_i^\zeta \leq 2\sqrt{\frac{2}{\rho\mathfrak{I}(i)}} \left(2\sigma\sqrt{14\log{\left(\frac{2a_0n\rho^2\mathfrak{I}(i)^2}{\delta}\right)}} + \rho\epsilon\right)\right\}\nonumber,\\
     \mathcal{N}(D,\pmb{\mu},\delta,\zeta) &\triangleq [n]\setminus \mathcal{H}(D, \pmb{\mu},\delta,\zeta) \nonumber
\end{align}
\normalsize
where $\Delta_i^{\zeta}\triangleq \max\{\Delta_i,\zeta\}$.
\end{definition}

\begin{lemma}\label{lem:zeta_connected_sample_complexity_proof}
Consider $n$-armed bandit problem with mean vector $\pmb{\mu}\in \mathbb{R}^n$. Let $G$ be a given connected similarity graph on the vertex set $[n]$, and further suppose that $\pmb{\mu}$ is $\epsilon$-smooth. Define 
\begin{align}
    T_{\text{sufficient}} &\ \triangleq  \sum_{i\in \mathcal{H}} \frac{1}{(\Delta_i^\zeta)^2}\left[c_1\log{\frac{c_2}{\delta\Delta_i^\zeta}} + \frac{\rho\epsilon}{2}\right] + \max_{i\in \mathcal{N}}\left\{\frac{2}{(\Delta_i^\zeta)^2}\left[c_1\log{\frac{c_2}{\delta\Delta_i^\zeta}} + \frac{\rho\epsilon}{2}\right]\right\}
\end{align}
where $\Delta_i^{\zeta}\triangleq \max\{\Delta_i,\zeta\}$. Then, with probability at least $1-\delta$, \algoname{}: (a) terminates in no more than $T_{\text{sufficient}}$ rounds, and (b) returns a $\zeta$-best arm
\end{lemma}


\begin{proof}
With out loss of generality, assume that $a^* = 1$. Let $\{t_i\}_{i=1}^n$ denote the number of plays of each arm upto time $T$. By Lemma~\ref{lem:variance_estimate}, we can state that,
\begin{align}
    \mathbb{P}\left(|\hat{\mu}^{i}_T - \mu_{i}| \geq \gamma_i(\pmb{\pi}_T)\right) \leq \frac{2\delta}{a_0nt_{\text{eff}, i}^2}
\end{align}
where, $\gamma_i(\pmb{\pi}_T) = \beta_i(\pmb{\pi}_T)\sqrt{t_{\text{eff}, i}^{-1}}$ and $\beta_i(\pmb{\pi}_T) = \left(2\sigma\sqrt{14\log{\left(\frac{2a_0nt_{\text{eff}, i}^2}{\delta}\right)}} + \rho\|\pmb{\mu}\|_G\right).$ 

As is reflected in the elimination policy~\eqref{eq:elimination_routine}, at any time $t$, arm 1 can be mistakenly eliminated in \algoname{} only if $\hat{\mu}_t^{i} > \hat{\mu}_t^1 + \gamma_i(\pmb{\pi}_t) +\gamma_1(\pmb{\pi}_t$). Let $T_s$ be the stopping time of \algoname{}, then the total failure probability for \algoname{} can be upper-bounded as,
\begin{align}
    \mathbb{P}(\text{Failure}) &\ \leq \sum_{t=2}^{T_s}\sum_{i=2}^n \mathbb{P}\left(\hat{\mu}_t^i \geq \hat{\mu}^1_t + \gamma_i(\pmb{\pi}_t) +\gamma_1(\pmb{\pi}_t)\right)\nonumber
\end{align}
Note that $\mathbb{P}\left(\hat{\mu}_t^i \geq \hat{\mu}^1_t + \gamma_i(\pmb{\pi}_t) +\gamma_1(\pmb{\pi}_t)\right) \leq \left[\mathbb{P}\left(\hat{\mu}_t^i \geq \mu^i + \gamma_i(\pmb{\pi}_t)\right) + \mathbb{P}\left(\hat{\mu}_t^1 \leq \mu^1 - \gamma_1(\pmb{\pi}_t)\right) \right]$, provided that $\gamma_i(\pmb{\pi}_t), \gamma_1(\pmb{\pi}_t)\leq \frac{\Delta_i^{\zeta}}{2}$. Hence the failure probability can be upperbounded as,
\begin{align}
    \mathbb{P}(\text{Failure}) &\ \leq  \sum_{i=2}^n \sum_{t=2}^{T_s} \left[\mathbb{P}\left(\hat{\mu}_t^i \geq \mu^i + \gamma_i(\pmb{\pi}_t)\right)  + \mathbb{P}\left(\hat{\mu}_t^1 \leq \mu^1 - \gamma_1(\pmb{\pi}_t)\right) \right]
\end{align}
conditioned on $\gamma_i(\pmb{\pi}_T), \gamma_1(\pmb{\pi}_T)\leq \frac{\Delta_i^\zeta}{2}$.

Let $a_0 \geq 4\sum_{t=1}^\infty t_{\text{eff}, i}^{-2}$, then from Lemma~\ref{lem:variance_estimate}, 
\begin{align}
    \mathbb{P}(\text{Failure}) &\ \leq \sum_{i=2}^n \sum_{t=2}^{T_s} \frac{2\delta}{a_0nt_{\text{eff}, i}^2}\nonumber\\
    &\ \leq \delta
\end{align}
The finiteness of the infinite sum of ${t_{\text{eff}, i}}^{-2}$ can be found in Lemma~\ref{lem:infinite_sum_bounded}. 

Thus, in order to keep $\mathbb{P}(\text{Failure}) \leq \delta$, it is sufficient if, at the time of elimination of arm $i$, we have enough samples to ensure,
\begin{align}\label{eq:elimination_criteria_1_zeta}
   \gamma_i(\pmb{\pi}_T) &\ \leq \frac{\Delta_i^\zeta}{2}\nonumber\\
   \sqrt{\frac{1}{t_{\text{eff}, i}}}\left(2\sigma\sqrt{14\log{\left(\frac{2a_0nt_{\text{eff}, i}^2}{\delta}\right)}} + \rho\epsilon\right) &\ \leq \frac{\Delta_i^\zeta}{2} 
\end{align}
Rewriting the above equation, 
\begin{align}\label{eq:suff_sample_a_i_1_zeta}
    \frac{\log{\left(a_i\right)}}{a_i} &\ \leq \sqrt{\frac{\delta}{d_1}}\frac{(\Delta_i^\zeta)^2}{d_0}
    \end{align}
where $d_0 = 64\times 14\sigma^2, d_1 = 2n a_0 e^{\frac{\rho^2\epsilon^2}{4\times14\sigma^2}}$ and $a_i = \sqrt{\frac{d_1}{\delta}}t_{\text{eff}, i}$. The following bound on $a_i$ is sufficient to satisfy eq.~\eqref{eq:suff_sample_a_i_1_zeta}, 
    \begin{align}
    a_i &\ \geq 2\sqrt{\frac{d_1}{\delta}}\frac{d_0}{(\Delta_i^\zeta)^2}\log{\left(\sqrt{\frac{d_1}{\delta}}\frac{d_0}{(\Delta_i^\zeta)^2}\right)}\nonumber
\end{align}
Resubstituting $t_{\text{eff}, i}$, we obtain the sufficient number of plays required to eliminate arm $i$ as, 
\begin{align}
    t_{\text{eff}, i} &\ \geq \frac{c_1}{(\Delta_i^\zeta)^2}\left[\log{\left(\frac{c_2}{\delta^{\frac{1}{2}}(\Delta_i^\zeta)^2}\right)} + c_3\right]
\end{align}
where $c_1 = 2 \times 64 \times 14\sigma^2$, $c_2 = 64 \times 14\sigma^2\sqrt{2na_0}$ and $c_3 = \frac{\rho^2\epsilon^2}{8 \times 14\sigma^2}$.


The further part of the proof depends crucially on the following bound on $t_{\text{eff}, i}$ for all $i\in [n]$ from Theorem~\ref{thm:effective_samples_proof} as follows:
\begin{align}
    t_{\text{eff}, i} \geq t_i +\frac{1}{2}\min\left\{\rho\mathfrak{I}(i), T-t_{i}\right\}
\end{align}
Hence a sufficiency condition for the~\algoname{} to produce the $\zeta$-best arm with probability $1-\delta$ is given when both the following conditions are satisfied,
\begin{align}
    t_i + \frac{\rho\mathfrak{I}(i)}{2} \geq \frac{1}{(\Delta_i^\zeta)^2}\left[c_1\log\left(\frac{c_2}{\delta\Delta_i^\zeta}\right) + \frac{\rho\epsilon}{2}\right]
\end{align}
and, 
\begin{align}
    T+t_i \geq T \geq  \frac{2}{(\Delta_i^\zeta)^2}\left[c_1\log\left(\frac{c_2}{\delta\Delta_i^\zeta}\right) + \frac{\rho\epsilon}{2}\right] 
\end{align}

From the Definition~\ref{def:b_d_set_def_full_gory_zeta} we have the set of competitive arms $\mathcal{H}$ and non-competitive arms $\mathcal{N}$ as follows:  
\begin{align}
    \mathcal{H} = &\ \left\{j\in [n] \big| \Delta_i^\zeta \leq 2\sqrt{\frac{2}{\rho\mathfrak{I}(i)}} \left(2\sigma\sqrt{14\log{\left(\frac{2a_0n\rho^2\mathfrak{I}(i)^2}{\delta}\right)}} + \rho\epsilon\right)\right\}
\end{align}
After the first $\max_{i\in \mathcal{N}}\left\{\frac{2}{(\Delta_i^\zeta)^2}\left[c_1\log{\frac{c_2}{\delta\Delta_i^\zeta}} + \frac{\rho\epsilon}{2}\right]\right\}$ samples, all arms in $\mathcal{N}$ are eliminated. Further, let $k_1$ be the index of the first arm to be eliminated (in $\mathcal{H}$) and $t^*_{k_1}$ be the number of samples of arm $k_1$ before getting eliminated then the total number of additional time steps played until the arm $k_1$ is eliminated is at most $|\mathcal{H}|t^*_{k_1}$. Let $k_2$ be the index of the next arm in $\mathcal{H}$ to be eliminated. The number of additional plays until the next arm is eliminated is given by $(|\mathcal{H}|-1)[t^*_{k_2} - t^*_{k_1}]$ and so on. 

Summing up all the samples required to converge to the optimal arm is given by, (let $t^*_{k_0} = 0$)
\begin{align}
    \sum_{h=1}^{|\mathcal{H}|}(|\mathcal{H}| - h))[t^*_{k_h} -t^*_{k_{h-1}}] = \sum_{h=1}^{|\mathcal{H}|-1}t^*_{k_h} = \sum_{i\in \mathcal{H}/1} t^*_i
\end{align}

Hence the final sample complexity can be computed as follows:
\begin{itemize}
    \item Number of plays required for arms in $\mathcal{H}$ :
    \begin{align}
        \sum_{i\in \mathcal{H}/1}t^*_i \geq \sum_{i\in \mathcal{H}/1}\frac{1}{(\Delta_i^\zeta)^2}\left[c_1\log{\frac{c_2}{\delta\Delta_i^\zeta}} + \frac{\rho\epsilon}{2}\right]
    \end{align}
    \item Number of plays required for all the arms in $\mathcal{N} := [n]/\mathcal{H}$ to be eliminated:
    \begin{align}
        T \geq \max_{i\in\mathcal{N}}\left\{\frac{2}{(\Delta_i^\zeta)^2}\left[c_1\log{\frac{c_2}{\delta\Delta_i^\zeta}} + \frac{\rho\epsilon}{2}\right]\right\}
    \end{align}
\end{itemize}

Hence the final sample complexity can be given by,
\begin{align}
    T_{\text{sufficient}} &\ \triangleq \max_{i\in \mathcal{N}}\left\{\frac{2}{(\Delta_i^\zeta)^2}\left[c_1\log{\frac{c_2}{\delta\Delta_i^\zeta}} + \frac{\rho\epsilon}{2}\right]\right\} + \sum_{i\in \mathcal{H}/1} \frac{1}{(\Delta_i^\zeta)^2}\left[c_1\log{\frac{c_2}{\delta\Delta_i^\zeta}} + \frac{\rho\epsilon}{2}\right]
\end{align}
\end{proof}

We extend Lemma~\ref{lem:zeta_connected_sample_complexity_proof} to the case when graph $G$ has disconnected clusters.


\begin{theorem}\label{thm: zeta_sample_complexity}
Consider $n$-armed bandit problem with mean vector $\pmb{\mu}\in \mathbb{R}^n$. Let $\mathcal{G}$ be the set of subgraphs given similarity graph $G$ on the vertex set $[n]$, and further suppose that $\pmb{\mu}$ is $\epsilon$-smooth. Define 
\begin{align}\label{eq:T_sample_generic_base}
    T_{\text{sufficient}} &\triangleq
    \min_{D\in \mathcal{G}}  \sum_{C\in\mathcal{C}_D}\left[\sum_{i\in C\cap\mathcal{H}_D} \frac{1}{(\Delta_i^\zeta)^2}\left[c_1\log{\frac{c_2}{\delta\Delta_i^\zeta}} + \frac{\rho\epsilon}{2}\right] \right. \nonumber\\ 
    & \left. + \max_{i\in C\cap\mathcal{N}_D}\left\{\frac{2}{(\Delta_i^\zeta)^2}\left[c_1\log{\frac{c_2}{\delta\Delta^\zeta_i}} + \frac{\rho\epsilon}{2}\right]\right\}\right]
\end{align}
where $\Delta_i^\zeta = \max\{\Delta_i,\zeta\}$ for all suboptimal arms, $\mathcal{H}_D$ and $\mathcal{N}_D$ are as in Definition~\ref{def:b_d_set_def_full_gory_zeta}, $\mathcal{C}_D$ is the set of connected components of subgraph $D \in \mathcal{G}$ and $c_1, c_2$ are constants independent of system parameters. Then, with probability at least $1-\delta$, \algoname{}: (a) terminates in no more than $T_{\text{sufficient}}$ rounds, and (b) returns a $\zeta$-best arm
\end{theorem}
\begin{proof}
From Lemma~\ref{lem:zeta_connected_sample_complexity_proof}, the sample complexity for each connected component $C\in \mathcal{C}$ can be given as,
\begin{align}\label{eq:intermediate_zeta_GRUB}
    T_{\text{sufficient}} &\ = \left[\sum_{i\in C\cap\mathcal{H}} \frac{1}{(\Delta_i^\zeta)^2}\left[c_1\log{\frac{c_2}{\delta\Delta_i^\zeta}} + \frac{\rho\epsilon}{2}\right]+ \max_{i\in C\cap\mathcal{N}}\left\{\frac{2}{(\Delta_i^\zeta)^2}\left[c_1\log{\frac{c_2}{\delta\Delta_i^\zeta}} + \frac{\rho\epsilon}{2}\right]\right\}\right]
\end{align}
where, summing it over all the components $C\in \mathcal{C}$, gives us the sample complexity for ~\algoname{} while considering graph $G$. 

Any subgraph $D$ of graph $G$ satisfies,
\begin{align}
\langle \pmb{\mu}, L_G\pmb{\mu}\rangle \leq \epsilon \Rightarrow \langle \pmb{\mu}, L_D\pmb{\mu}\rangle \leq \epsilon
\end{align}



As seen in Definition~\ref{def:b_d_set_def_full_gory_zeta}, the influence factor is instrumental in deciding the competitive and non-competitive sets, which further dictates the sample complexity bounds. Further, notice from Lemma~\ref{lem:d_change_subgraph} that the influence factor $\mathfrak{I}(i, D)$ is not monotonic when considering subgraph $D$ of graph $G$.  Hence considering a subgraph of $G$ could potentially increase the number of non-competitive arms and provide us with a tighter bound on the performance for \algoname{}. 

Hence $T_{\text{sufficient}}$ can be made tighter by considering the minimum value over the entire set of subgraphs $\mathcal{G}$. 
\end{proof}
Note that, as in the case of \algoname{}, the $\zeta$-\algoname{} algorithm's performance \emph{automatically} adapts to the best possible subgraph in $\mathcal{G}$.

\section{The Incomparability of the Graph Bandits problem with Linear Bandits}\label{app:as_linear_bandit}

In this section, we provide toy example as well as theoritical base to show the difference between the framework of bandits with graph side information and linear bandits. In this appendix, we first explain the working of the toy example in more detail and then head towards the proof of proposition.

\subsection{Toy example}\label{subsec:toy}
Consider  $3$-armed bandit problem with graph side information : Let graph $G$ encodes the similarity relation between the mean values of the three arms, i.e. \begin{align}\label{eq:toy_graph}
    \langle \pmb{\mu}, L_G\pmb{\mu}\rangle \leq \epsilon
\end{align} for some constant $\epsilon > 0$. Let $E_G$ denote the edge set of graph $G$ and $\mathbbm{1}_{(1,2)}, \mathbbm{1}_{(2,3)}$ and $\mathbbm{1}_{(1,3)}$ encodes the event if edges $\{(1,2), (2.3), (1,3)\}\in E_G$ are present in graph $G$. For the sake of a non-trivial analysis we take that either $(1,3)$ or $(2,3)$ is present in $E_G$ (alternate case is argued later). 

We can write equation~\eqref{eq:toy_graph} as,
\begin{align}\label{eq:eps_big}
    \mathbbm{1}_{(1,2)}(\mu_1-\mu_2)^2 + \mathbbm{1}_{(2,3)}(\mu_2-\mu_3)^2 + 
    \mathbbm{1}_{(1,3)}(\mu_1-\mu_3)^2 \leq \epsilon
\end{align}
In order to compare the dependence behaviour of $\mu_3$ on $\mu_1, \mu_2$ we can rearrange the above as,
\begin{align}\label{eq:mu3_quadratic}
    &\ \mu_3^2 \left(\mathbbm{1}_{(2,3)} + 
    \mathbbm{1}_{(1,3)}\right) -2\mu_3\left(\mathbbm{1}_{(2,3)}\mu_2 + 
    \mathbbm{1}_{(1,3)}\mu_1\right) \nonumber\\ &\ + \left(\mathbbm{1}_{(2,3)}\mu_2^2 + 
    \mathbbm{1}_{(1,3)}\mu_1^2 + \mathbbm{1}_a(\mu_1-\mu_2)^2 - \epsilon\right) \leq 0
\end{align}
Looking at equation~\eqref{eq:mu3_quadratic} as a quadratic in $\mu_3$ and finding the solutions, we obtain that,
\begin{align}
    &\ \mu_3\geq \frac{\left(\mathbbm{1}_{(2,3)}\mu_2 + 
    \mathbbm{1}_{(1,3)}\mu_1\right) }{\left(\mathbbm{1}_{(2,3)} + 
    \mathbbm{1}_{(1,3)}\right)}\nonumber\\
    &\ - \frac{ \sqrt{\left(\mathbbm{1}_{(2,3)}\mu_2 + 
    \mathbbm{1}_{(1,3)}\mu_1\right)^2 - \left(\mathbbm{1}_{(2,3)} + 
    \mathbbm{1}_{(1,3)}\right)\left(\mathbbm{1}_{(2,3)}\mu_2^2 + 
    \mathbbm{1}_{(1,3)}\mu_1^2 + \mathbbm{1}_{(1,2)}(\mu_1-\mu_2)^2 - \epsilon\right) }}{\left(\mathbbm{1}_{(2,3)} + 
    \mathbbm{1}_{(1,3)}\right)}\nonumber\\
    &\ \mu_3\leq \frac{\left(\mathbbm{1}_{(2,3)}\mu_2 + 
    \mathbbm{1}_{(1,3)}\mu_1\right) }{\left(\mathbbm{1}_{(2,3)} + 
    \mathbbm{1}_{(1,3)}\right)}\nonumber\\
    &\ + \frac{ \sqrt{\left(\mathbbm{1}_{(2,3)}\mu_2 + 
    \mathbbm{1}_{(1,3)}\mu_1\right)^2 - \left(\mathbbm{1}_{(2,3)} + 
    \mathbbm{1}_{(1,3)}\right)\left(\mathbbm{1}_{(2,3)}\mu_2^2 + 
    \mathbbm{1}_{(1,3)}\mu_1^2 + \mathbbm{1}_{(1,2)}(\mu_1-\mu_2)^2 - \epsilon\right) }}{\left(\mathbbm{1}_{(2,3)} + 
    \mathbbm{1}_{(1,3)}\right)}\nonumber\\
\end{align}

Further simplifying it, we get the following:
\begin{align}\label{eq:bounds_simplified}
    &\ \mu_3\geq \frac{\left(\mathbbm{1}_{(2,3)}\mu_2 + 
    \mathbbm{1}_{(1,3)}\mu_1\right) }{\left(\mathbbm{1}_{(2,3)} + 
    \mathbbm{1}_{(1,3)}\right)}\nonumber\\
    &\ -\frac{ \sqrt{\epsilon\left(\mathbbm{1}_{(2,3)} + 
    \mathbbm{1}_{(1,3)}\right) - \mathbbm{1}_{(2,3)}\mathbbm{1}_{(1,3)}(\mu_2-\mu_1)^2 -  \mathbbm{1}_{(1,2)}\left(\mathbbm{1}_{(2,3)} + 
    \mathbbm{1}_{(1,3)}\right)(\mu_1-\mu_2)^2 }}{\left(\mathbbm{1}_{(2,3)} + 
    \mathbbm{1}_{(1,3)}\right)}\nonumber\\
    &\ \mu_3\geq \frac{\left(\mathbbm{1}_{(2,3)}\mu_2 + 
    \mathbbm{1}_{(1,3)}\mu_1\right) }{\left(\mathbbm{1}_{(2,3)} + 
    \mathbbm{1}_{(1,3)}\right)}\nonumber\\
    &\ +\frac{ \sqrt{\epsilon\left(\mathbbm{1}_{(2,3)} + 
    \mathbbm{1}_{(1,3)}\right) - \mathbbm{1}_{(2,3)}\mathbbm{1}_{(1,3)}(\mu_2-\mu_1)^2 -  \mathbbm{1}_{(1,2)}\left(\mathbbm{1}_{(2,3)} + 
    \mathbbm{1}_{(1,3)}\right)(\mu_1-\mu_2)^2 }}{\left(\mathbbm{1}_{(2,3)} + 
    \mathbbm{1}_{(1,3)}\right)}\nonumber\\
\end{align}
The above equation leads to non-trivial bounds as $\epsilon$ satisfies equation~\eqref{eq:eps_big}. 

For the simple case of when edge $(2,3)$ is present and $(1,3)$ is not, the above analysis simplifies to,
\begin{align}
    \mu_2 - \sqrt{\epsilon - \mathbbm{1}_{(1,2)}(\mu_1 - \mu_2)^2}\leq \mu_3 \leq \mu_2 + \sqrt{\epsilon - \mathbbm{1}_{(1,2)}(\mu_1 - \mu_2)^2}
\end{align}
Similar analysis can be made when edge $(1,3)$ is present and $(2,3)$ is not. 

This shows that even with the complete knowledge of $\mu_1, \mu_2$ we can only estimate $\mu_3$ to a interval $[\mu_{\text{low}}, \mu_{\text{high}}]$ where the endpoints of interval are given by equation~\eqref{eq:bounds_simplified}. For the case when neither of the edges $(1,3)$ or $(2,3)$ are present (i.e. $\mathbbm{1}_{b,c} = 0$), then $\mu_3 \in [-\infty, \infty]$ even with the full knowledge of $\mu_1, \mu_2$ as there is no relation between the means of arm $3$ to arm $1,2$, this is also reflected in the equation~\eqref{eq:bounds_simplified}

\noindent
We first formally rewrite the two setups:
\subsubsection*{Bandits with graph side information}

Consider an $n$-armed linear bandit problem, each arm $i\in [n]$ is associated with a mean vector $\mathbf{\mu}\in \mathbb{R}^n$, where $\mu_i$ corresponds to the mean value of arm $i$. We are provided with further information using a graph $G$ that $\langle \pmb{\mu}, L_G\pmb{\mu}\rangle \leq \epsilon$, where $\epsilon>0$. In each round $t$, the learner chooses some arm $i\in [n]$ and observes the reward $y_t = \mu_i + \eta_t$, where $\eta_t$ is a subgaussian random noise with $\sigma^2$ variance. Denote the arm with the best mean reward with $i^*$, i.e. $i^* = \mathop{\arg\max}_{i\in[n]} \mu_i$. The goal of the learner is to to output the index of the arm $i^*$ with probability $1-\delta$, $\delta>0$ in as few samples as possible. 


\subsubsection*{Linear bandits}
Consider an $n$-armed linear bandit problem, each arm $i\in [n]$ is associated with a feature vector $\mathbf{x}_i\in \mathbb{R}^d$, where $d$ can be lower than $n$. In each round $t$, the learner chooses an action $\textbf{a}_t = \mathbf{x}_i$ for some $i\in [n]$ and observes the reward $y_t = \langle \mathbf{a}_t, \pmb{\theta}\rangle + \eta_t$, where $\pmb{\theta}\in \mathbb{R}^d$ is an unknown parameter and the $\eta_t$ is a subgaussian random noise with $\sigma^2$ variance. Denote the arm with the best mean reward with $i^*$, i.e. $i^* = \mathop{\arg\max}_{i\in[n]} \langle \mathbf{x}_i, \pmb{\theta}\rangle$. The goal of the learner is to to output the index of the arm $i^*$ with probability $1-\delta$, $\delta>0$ in as few samples as possible. 

\subsubsection*{Graph vs Linear Bandit framework}
In this section we address the question of whether the $n$ armed bandit problem, with the additional information of $\langle \pmb{\mu}, L\pmb{\mu}\rangle \leq \epsilon$ can be solved using a linear bandits framework. The metric we use for such a comparison is the set of $n$-armed bandit problems i.e. set of $\pmb{\mu}$ which can be expressed once the parameters of the two frameworks are fixed. For the case of linear bandits this would be the lower dimension $k$ and feature vector $\mathbf{a}$ corresponding to the reward and for graph bandit framework this indicates the graph $G$ and $\epsilon$. Let the set of problems addressed by linear bandit framework be denotes by $L_{\mathbf{a}}$ and that by the graph bandits framework denoted by $L_{G, \epsilon}$. We prove that $L_{\mathbf{a}}$ and $L_{G, \epsilon}$ represents sets with fundamentally different properties. Hence we prove that the set of problems addressed by linear bandits and the proposed graph bandit framework of this paper address fundamentally different as there cannot exist one-to-one mapping between the two.

We can further provide additional arguments for the case when $\langle \pmb{\mu}, L_G\pmb{\mu}\rangle = 0$. For this, we demonstrate an example graph bandit problem that is cast as a linear bandit to reveal the incomparability of these frameworks. 

Firstly, a $n$-armed bandit problem without any graph can be easily seen as linear bandits by associating the canonical basis for $\mathbb{R}^n$ $\{\mathbf{e}_i\}_{i=1}^n$ as the feature vectors and the mean vector $\pmb{\mu}\in \mathbb{R}^n$ as the unknown reward vector. This provides up with the mean reward function for arm $i\in [n]$ as $\langle \mathbf{e}_i, \pmb{\mu}\rangle = \mu_i$. 

In order to cast the graph bandit problem in a linear bandit framework, we need to associate every arm index $i$ with a feature vector $\mathbf{x}_i$ and identify the unknown feature vector $\pmb{\theta}$ for the problem. We achieve this by modifying the feature vectors $\{\mathbf{e}_i\}_{i=1}^n$ and the reward vector $\pmb{\mu}$ based on the graph Laplacian $L_G$. 

Following is the information available at hand in the current graph bandit problem: we are provided with an $n$-armed bandit with an unknown mean vector $\pmb{\mu}$ smooth on a graph $G$, i.e. $\langle \pmb{\mu}, L_G\pmb{\mu}\rangle \leq \epsilon$. For this toy problem, we consider the  graph $G$ to be connected. 


Let $\{\pmb{\nu}_i\}_{i=1}^n$ and $0 = \lambda_1 < \dots < \lambda_n$ denote the eigenvectors and eigenvalues of the Laplacian $L_G$ respectively. It can be easily seen that $\pmb{\mu} = \sum_{i=1}^n a_i\pmb{\nu}_i$ for some $a_i \geq 0~~\forall i\in [n]$. The reward function of arm $j$ is $$\langle\mathbf{e}_j, \pmb{\mu}\rangle = \sum_{i=1}^n a_i\langle\mathbf{e}_j, \pmb{\nu}_i\rangle = a_1 + \sum_{i=2}^n a_i\langle \mathbf{e}_j, \pmb{\nu}_i \rangle$$ 
the second equality follows from the properties of graph Laplacian we know that $\pmb{\nu}_1 = \mathbbm{1}_n$, is the only eigenvector associated to 0 eigenvalue in a connected graph.

Without loss of generality we can assume $a_1=0$ as $a_1$ does not depend on the arm index $j$. Notice that letting $a_1 =0$ is equivalent to having $\sum_{i=1}^n\mu_i =0$. Also, the graph constraint can be rewritten as follows: $$\langle\pmb{\mu}, L_G\pmb{\mu} \rangle\leq \epsilon \Rightarrow \sum_{i=1}^n \lambda_ia_i^2 = \langle \pmb{\theta}, \pmb{\theta}\rangle = \|\pmb{\theta}\|_2^2\leq \epsilon$$ where $\pmb{\theta} = (\sqrt{\lambda_1}a_1, \dots, \sqrt{\lambda_n}a_n)$.

Using the above we can cast the graph bandit problem as the linear bandit problem with the mean reward function of arm $j$ expressed as $$\langle \mathbf{e}_j, \pmb{\mu}\rangle = \sum_{i=2}^n \frac{\theta_i}{\sqrt{\lambda_i}}\langle \mathbf{e}_j, \nu_i\rangle = \langle \mathbf{x}_j, \pmb{\theta}\rangle$$

Hence, the new linear bandit problem is such that the set of arms is $\{\mathbf{x}_j\}_{j=1}^n$, the unknown parameter is a vector $\pmb{\theta}$, the expected reward of an arm is $\langle \mathbf{x}_j, \pmb{\theta}\rangle$ and the unknown parameter satisfies the constraint $\|\pmb{\theta}\|_2^2 \leq \epsilon$. 

We discuss below the drawbacks of casting a graph bandit problem into a linear bandit framework: 
\begin{itemize}
    \item The original best-arm identification is an $n$-armed problem and the recasted linear bandit problem still has feature vectors with dimensionality $n$ and hence no low-dimensional benefit of linear bandits is completely lost. Having a performance bound for any algorithm for linear bandits which scales in $n$, the number of arms gives us no additional advantage.  
    \item The above conversion to linear bandit setup only works when the graph $G$ is connected. Recasting problem setup with disconnected components require assumption of $\sum_{i\in C} \mu_i = 0$ on individual connected components, which is unrealistic. The results of \algoname{} holds with or without this assumption. 
    \item Consider the corner case of $\epsilon = 0$, the linear bandit problem setup derived becomes that of $\mathop{\arg\max}_{i}\langle \mathbf{x}_i, \pmb{\theta}\rangle$ such that $\|\pmb{\theta}\| \leq 0$ which is only possible if $\|\pmb{\theta}\| =0$ and in this case we can observe two interesting facts:
    \begin{itemize}
        \item If the graph $G$ is completely connected then the problem is trivial, since $$\epsilon = 0 \Rightarrow \langle \pmb{\mu}, L_G\pmb{\mu}\rangle = 0 \Rightarrow (\mu_i -\mu_j)^2 =0~~\forall i, j\in [n], i\neq j$$
        This implies all arms are equal and optimal and the solution is trivial. Here the mean reward function of all arms $i$ is $\langle \mathbf{x}_i, \pmb{\theta}\rangle = 0$ since $\theta =0 $ and hence gives the correct output (any arm $i$).
        \item Suppose graph $G$ has two connected components $C_1, C_2$, where $C_k$ indicates the arm indices in the connected component $k$. Further assume that $\mu_i = 1~~\forall i \in C_1, \mu_i=-1 ~~\forall i \in C_2$. Considering the case of $\epsilon = 0$ here gives us the following :
        $$\epsilon = 0 \Rightarrow \langle \pmb{\mu}, L_G\pmb{\mu}\rangle = 0 \Rightarrow (\mu_i -\mu_j)^2 =0~~\forall i\neq j, i, j\in C_k, k=1,2 $$
        Here the mean reward function of all arms $i$ is $\langle \mathbf{x}_i, \pmb{\theta}\rangle = 0$ since $\theta =0 $ but this is incorrect since not all arms are optimal. 
    \end{itemize}
    Our graph bandit setup and the performance of \algoname{} is independent of all of these drawbacks and provides us with a better sample complexity than 
    vanilla best arm identification algorithms.
\end{itemize} 

We solidify these arguments with the following Propositions.

\begin{proposition}\label{prop:graph_vs_linear_thm_proof}
Consider $n$-armed bandit setup parameterized by mean vector $\pmb{\mu}\in \mathbb{R}^n$. Given graph $G$ and $\epsilon >0$, let $D_{G, \epsilon} := \{\pmb{\mu}\in \mathbb{R}^n |~~ \langle \pmb{\mu}, L_G\pmb{\mu}\rangle \leq \epsilon\}$ represent a subset of bandit problems in $\mathbb{R}^n$,  $L_G$ denoting the laplacian matrix corresponding to graph $G$. Then $m(D_{G, \epsilon})>0$ where $m(\cdot)$ is the Lebesgue measure on $\mathbb{R}^n$.
\end{proposition}

\textit{Sketch of proof :} We solidify the intuition from toy example~\ref{subsec:toy} to show the distinction in the two frameworks using measure theoretic argument. We split the $n$-armed bandit problem with graph side-information into two complementary scenarios:

(a) $L_{>}: \{\pmb{\mu}\in \mathbb{R}^n |~~ 0 < \langle \pmb{\mu}, L_G\pmb{\mu}\rangle \leq\epsilon, \epsilon> 0\}$

(b) $L_{=}:\{\pmb{\mu}\in \mathbb{R}^n |~~  \langle \pmb{\mu}, L_G\pmb{\mu}\rangle  = 0\}$

We prove that the set $D_{G, \epsilon} = L_{>}\cup L_{=}$ has fundamentally different measure theoretic properties than $D_{\pmb{\theta}}$ (linear bandits framework) and hence the two problems setups tackle completely different domain of questions.

\begin{theorem}\label{thm:graph_vs_linear_>0}
Consider $n$-armed bandit setup. Let $D_{G, \epsilon}, D_{\pmb{\theta}}$ represent the subset of problems in $\mathbb{R}^n$ as follows:
\begin{align}
    D_{G, \epsilon} &\ = \{\pmb{\mu}\in \mathbb{R}^n |~~ 0 < \langle \pmb{\mu}, L_G\pmb{\mu}\rangle \leq \epsilon, \epsilon>0\}\nonumber\\
    D_{\pmb{\theta}} &\ = \{\pmb{\mu}\in \mathbb{R}^n |~~\mu_i = \langle\mathbf{a}_i,\pmb{\theta}\rangle, ~~\mathbf{a}_i\in \mathbb{R}^k,~~\forall i\in [n]\}\nonumber
\end{align}
where $k < n, \pmb{\theta}\in \mathbb{R}^k$ indicates the reward vector and $L_G$ is the laplacian matrix corresponding to graph $G$. Then $D_{G, \epsilon} \not\subset D_{\pmb{\theta}}$, $D_{\pmb{\theta}} \not\subset D_{G, \epsilon}, m(D_{G, \epsilon}) >0$ and $m(D_{\pmb{\theta}})=0$ where $m(\cdot)$ is the Lebesgue measure on $\mathbb{R}^n$.
\end{theorem}
\begin{proof}
The two problem subset definitions represents the following :

a) $D_{G, \epsilon}$ -- Given a graph $G$ and violation parameter $\epsilon$, $D_{G, \epsilon}$ represents mean-reward vectors $\pmb{\mu}\in\mathbb{R}^n$ which satisfy the graph bandit setup. 

b) $D_{\pmb{\theta}}$ -- Given lower dimension $k<n$ and the corresponding reward vector $\pmb{\theta}$, the set $D_{\pmb{\theta}}$ indicates the set of all the mean rewards $\{\mu_i\}_{i=1}^n$ for $n$-armed bandit setup such that the mean-reward vector can be represented by $k$-dimensional feature vectors.

First consider the following arguments :
\begin{itemize}
    \item Notice that if $\pmb{\alpha}, \pmb{\beta} \in D_{\pmb{\theta}}$ then $c_1\pmb{\alpha}+c_2\pmb{\beta} \in D_{\pmb{\theta}}$ and $\mathbf{0}\in D_{\pmb{\theta}}$, where $\mathbf{0}$ is the all zero vector in $\mathbb{R}^n$. Hence we can conclude $D_{\pmb{\theta}}$ is a subspace of $\mathbb{R}^n$. Since all the elements of the set $D_{\pmb{\theta}}$ can be as a linear map to a $k$-dimensional subspace constructed out of $\{\mathbf{a}_i\}_{i=1}^n$, $\mathbf{a}_i\in \mathbb{R}^k$ for all $i\in [n]$, hence $D_{\pmb{\theta}}$ is a $k$-dimensional subspace of $\mathbb{R}^n$. Accordingly, $m(D_{\pmb{\theta}}) = 0$ where $m$ is the Lebesgue measure on the euclidean space $\mathbb{R}^n$.
    \item Consider the set $D_{G, \epsilon}$ and $\pmb{\mu}$ such that $\langle \pmb{\mu}, L_G\pmb{\mu}\rangle = 0$ (existence of such a $\pmb{\mu}$ is easy to prove by making $\mu_i=\mu_j$ for every edge in $G$). Given that $\epsilon >0$, $\exists \delta >0$ such that $\forall \pmb{\sigma}\in \mathcal{B}(\mathbf{0}, \delta)$, 
    \begin{align}
        \langle (\pmb{\mu}+\pmb{\sigma}), L_G(\pmb{\mu} + \pmb{\sigma})\rangle &\ = \sum_{\{i,j\}\in E_G}A_{ij} (\mu_i +\sigma_i - \mu_j -\sigma_j)^2\nonumber\\ 
        &\ = \sum_{\{i,j\}\in E_G}A_{ij} (\sigma_i -\sigma_j)^2~~~~~~~~~~~(\langle \pmb{\mu}, L_G\pmb{\mu}\rangle = 0\Rightarrow \mu_i = \mu_j~~\forall (i,j)\in E_G)\nonumber\\
        &\ \leq \|A_{G}\|_{\infty}\sum_{\{i,j\}\in E_G}(\sigma_i - \sigma_j)^2\nonumber\\
        &\ \leq 4\|A_G\|_{\infty}\|\pmb{\sigma}\|^2_2
        &\ 
    \end{align}
    Taking $\delta < \frac{\epsilon}{4\|A_G\|_{\infty}}$ proves that $\forall \pmb{\sigma}\in \mathcal{B}(\mathbf{0}, \delta)$, $\langle (\pmb{\mu}+\pmb{\sigma}), L_G(\pmb{\mu} + \pmb{\sigma})\rangle \leq \epsilon$. Hence $\mathcal{B}(\mathbf{0}, \delta) \subset D_{G, \epsilon}$ implying $m(D_{G, \epsilon}) > m(\mathcal{B}(\mathbf{0}, \delta)) = \delta^n$. 
    
    Further, consider $\pmb{\theta}\in \mathbb{R}^n$ such that $\langle \pmb{\theta}, L_G\pmb{\theta}\rangle = \epsilon$ for some $\epsilon >0$ hence $\pmb{\theta}\in D_{G, \epsilon}$. Then it is easy to see that $2\pmb{\theta}\not\in D_{G, \epsilon}$ as $\langle 2\pmb{\theta}, L_G(2\pmb{\theta})\rangle = 4\epsilon > \epsilon$. 
\end{itemize}

We can thus conclude that $D_{\pmb{\theta}}$ is a $k$-dimensional subspace of $\mathbb{R}^n$ which is a measure zero set and $D_{G, \epsilon}$ is a positive measure set which is not closed under multiplication. Hence we can easily see that $D_{G, \epsilon} \not\subset D_{\pmb{\theta}}$, $D_{\pmb{\theta}} \not\subset D_{G, \epsilon}$.
\end{proof}

\begin{theorem}\label{thm:graph_vs_linear_proof}
Consider $n$-armed bandit setup. Let $D_{G, \epsilon}, D_{\pmb{\theta}}$ represent the subset of problems in $\mathbb{R}^n$ as follows:
\begin{align}
    D_{G, \epsilon} &\ = \{\pmb{\mu}\in \mathbb{R}^n |~~ \langle \pmb{\mu}, L_G\pmb{\mu}\rangle \leq \epsilon, \epsilon>0\}\nonumber\\
    D_{\pmb{\theta}} &\ = \{\pmb{\mu}\in \mathbb{R}^n |~~\mu_i = \langle\mathbf{a}_i,\pmb{\theta}\rangle, ~~\mathbf{a}_i\in \mathbb{R}^k,~~\forall i\in [n]\}\nonumber
\end{align}
where $k < n, \pmb{\theta}\in \mathbb{R}^k$ indicates the reward vector and $L_G$ is the laplacian matrix corresponding to graph $G$. Then $D_{G, \epsilon} \not\subset D_{\pmb{\theta}}$, $D_{\pmb{\theta}} \not\subset D_{G, \epsilon}, m(D_{G, \epsilon}) >0$ and $m(D_{\pmb{\theta}})=0$ where $m(\cdot)$ is the Lebesgue measure on $\mathbb{R}^n$.
\end{theorem}

\begin{proof}

We can split the argument into two parts:
\begin{itemize}
    \item $\langle\pmb{\mu}, L_G\pmb{\mu}\rangle >0$
    \item $\langle\pmb{\mu}, L_G\pmb{\mu}\rangle =0$
\end{itemize}

Theorem~\ref{thm:graph_vs_linear_>0} addresses the first part of the argument

For the case when second part, i.e. $\langle \pmb{\mu}, L_G\pmb{\mu}\rangle = 0$, let $L_{=}:\{\pmb{\mu}\in \mathbb{R}^n |~~  \langle \pmb{\mu}, L_G\pmb{\mu}\rangle  = 0\}$. Note that for any $\pmb{\mu}\in \mathbb{R}^n$ only happens if and only if $\pmb{\mu}\in \mathcal{N}(L_G)$ where $\mathcal{N}(\cdot)$ represents the null space of the matrix. Since $L_G$ is rank deficient $L_{=}$ is a set in a subspace of $\mathbb{R}^n$ and hence $m(L_{=}) = 0$. 

Thus $m(L_> \cup L_=) >0$ which is fundamentally different from linear bandits which addresses problems of measure zero.

\end{proof}

Thus we can conclude that the two frameworks of graph and linear bandits address fundamentally different domain of problems.
\section{Supporting Results}\label{app:support2}
This appendix is devoted to providing supporting results for many of the theorems and lemmas in the paper. 

\subsection{Notation and Definition}
Let $\{t_i(T)\}_{i=1}^n$ (denoted as $\{t_i\}_{i=1}^n$ for ease of reading) indicate the number of plays of each arm until time $T$.
Let $X\in \mathbb{R}^{n\times n}$ be a matrix, then $\{\lambda_i(X)\}_{i=1}^n$ indicate the eigenvalues of matrix $X$ in an increasing order.

Let $N(\pmb{\pi}_T) = \sum_{t=1}^T \mathbf{e}_{\pi_t}\mathbf{e}_{\pi_t}^T$ be the diagonal counting matrix. Note that $N(\pmb{\pi}_T)$ can be written as $N(\{t_i\}_{i=1}^n)$ since the diagonal counting matrix only depends on the number of plays of each arm, rather than the each sampling sequence $\pmb{\pi}_T$. 

We next establish some properties of the influence function $\mathfrak{I}$.  
\begin{lemma}\label{lem:d_diff_forms}
    Let $D$ be an arbitrary graph with $n$ nodes and let $\{t_i\}_{i=1}^n$ be the number of times all arms are sampled till time $T$. For each node $j\in [n]$, the following are equivalent:
\begin{align}\label{eq:d}
 \frac{1}{\mathfrak{I}(j, D)} &\ = \underset{\sum_{i\in D_j, i\neq j} t_i = T}{\max} \left\{[K(i, D)]_{jj}\right\}~~~~~(A)\nonumber\\
    &\ = \underset{k\in D_j,\sum_{i\in D_j, i\neq j} t_i = T}{\max} \left\{[V_j(  \{t_i\}_{i\in D_j}, D)^{-1}]_{jj}  - [V_j(  \{t_i\}_{i\in D_j}, D)^{-1}]_{kk}\right\}~~~~~(B)\nonumber\\
    &\ = \underset{\sum_{i\in D_j, i\neq j} t_i = T}{\max} \left\{[V_j(  \{t_i\}_{i\in D_j}, D)^{-1}]_{jj} - \underset{k\in D_j}{\min}[V_j(\{t_i\}_{i\in D_j}, D)^{-1}]_{kk}\right\}~~~~~(C)\nonumber\\
    &\ = \underset{\sum_{i\in D_j, i\neq j} t_i = T}{\max} \left\{[V_j(  \{t_i\}_{i\in D_j}, D)^{-1}]_{jj} - \frac{1}{T}\right\}~~~~~(D)
\end{align}
where $K(i, D)$ be defined as in Definition~\ref{def:d_better}
\end{lemma}
\begin{proof}
Let $f(\cdot, \cdot)$ denote the following:
\begin{align}
    f(i, D) = \underset{\sum_{i\in D_j, i\neq j} t_i = T}{\max} \left\{[K(i, D)]_{jj}\right\}\nonumber
\end{align}
We prove the rest by showing equivalence between $(A), (B), (C)$ and $(D)$. 
\begin{itemize}
\item $(A)\Leftrightarrow (D)$ : A simple extension of Lemma~\ref{lem:V_indep_T} to the case of disconnected clustered graph $D$, $\forall \pmb{\pi}_T \in \mathcal{U}(T, D_j)$ we obtain,
\begin{align}\label{eq:disc_v_t_indep_t}
    V_j( \pmb{\pi}_T, D)^{-1} = \frac{1}{T}\mathbb{1}\mathbb{1}^T + K(\pi_1, D)
\end{align}
where $K(\pi_1, D)$ is as defined in Definition~\ref{def:d_better}. 
Thus, we have the equivalence by explicitly writing the diagonal element of eq~\eqref{eq:disc_v_t_indep_t}, 
\begin{equation}
    [V_j( \pmb{\pi}_T, D)^{-1}]_{jj} - \frac{1}{T} = [K(\pi_1, D)]_{jj}
\end{equation}
Hence we have the equivalence as,
\begin{align}
    f(i, D) = \underset{\sum_{i\in D_j, i\neq j} t_i = T}{\max} \left\{[V_j(  \{t_i\}_{i\in D_j}, D)^{-1}]_{jj} - \frac{1}{T}\right\}
\end{align}
\item $(C) \Leftrightarrow (D)$ : Let $\{t^*_i\}_{i\in D_j}$ denote the following:
\begin{align}
    \{t_i^*(j)\}_{i\in D_j} \in \underset{\sum_{i\in D_j, i\neq j} t_i = T}{\arg\max} \left\{[V_j(  \{t_i\}_{i\in D_j}, D)]^{-1}_{jj}- \frac{1}{T}\right\}
\end{align} 
From Lemma~\ref{lem:optima_simplex_proof}, the optimal $\{t^*_i(j)\}_{i\in D_j}$ occurs in $\mathcal{U}(j, T)$, i.e. $\exists \{t^*_i(j)\}_{i\in D_j}$ such that $t_l^*(j) = T$ and $t_k^*(j) = 0~~\forall k\neq l$ for some $l\in D_j$. Further by Lemma~\ref{lem:min_v_t}, 
\begin{align}
    \min_{k\in D_j}[V_j(  \{t_i\}_{i\in D_j}, D)^{-1}]_{kk} = \frac{1}{T}
\end{align}
Hence $\{t_i^*(j)\}_{i\in D_j}$ is also a solution for the following problem:
\begin{align}
    \{t_i^*(j)\}_{i\in D_j} &\ \in \underset{\sum_{i\in D_j, i\neq j} t_i = T}{\arg\max} \left\{[V_j(  \{t_i\}_{i\in D_j}, D)]^{-1}_{jj} \right.\nonumber\\ &\ \left.- \min_{k\in D_j}[V_j(  \{t_i\}_{i\in D_j}, D)^{-1}]_{kk}\right\}
\end{align}
Hence we can conlcude that,
\begin{align}
    f(i, D) &\ = \underset{\sum_{i\in D_j, i\neq j} t_i = T}{\max} \left\{[V_j(  \{t_i\}_{i\in D_j}, D)]^{-1}_{jj}\right. \nonumber\\ &\ \left. - \min_{k\in D_j}[V_j(  \{t_i\}_{i\in D_j}, D)^{-1}]_{kk}\right\}
\end{align}
\item $(B) \Leftrightarrow (C)$ :  

Note that $\underset{k\in D_j,\sum_{i\in D_j, i\neq j} t_i = T}{\max}[V_j(\{t_i\}_{i\in D_j}, D)^{-1}]_{jj}]$ does not depend on arm node index $k\in D_j$. Hence, the equivalence follows.
\end{itemize}
The resistance distance $r(i, j)$ Definition~\ref{def:res_dis} is independent of $\delta$ for all $i,j\in [n]$ (The addition of diagonal elements and subtraction of off diagonal elements removes the dependence on $\delta$~\cite{bapat2010resistance}). 

Note that $V_T = N_T+\rho L_G$, hence $V_T^{-1}$ gives the psuedo-inverse of the Laplacian matrix for graph $G$. We show in Lemma~\ref{lem:optima_simplex_proof} that the matrix $R$ (denoting as $R(\delta)$ to explicitly show dependence on $\delta$) linked with $V_T^{-1}$ is independent of number of samples $T$. Since both matrix $R$ and $V_T$ are psuedo-inverse of the Laplacian $L_G$. Thus we can conclude the following :
\begin{align}
    \lim_{\delta \rightarrow 0} [R(\delta)]_{ij} -\frac{1}{\delta} = \lim_{T \rightarrow 0} [V(\{t_i\}_{i=1}^n, G)^{-1}]_{ij} - \frac{1}{T}
\end{align}
where $T\rightarrow 0 $ implies $t_i\rightarrow 0 ~~~\forall i\in [n]$. Further, 
\begin{align}\label{eq:crude_new_res_dis}
    &\ \lim_{\delta\rightarrow 0} R(\delta)_{ii} + R(\delta)_{jj} - R(\delta)_{ij} - R(\delta)_{ji} \nonumber\\ &\ = \lim_{T\rightarrow 0} [V(\{t_i\}_{i=1}^n, G)^{-1}]_{ii} + [V(\{t_i\}_{i=1}^n, G)^{-1}]_{jj} \nonumber\\ &\ - [V(\{t_i\}_{i=1}^n, G)^{-1}]_{ij} - [V(\{t_i\}_{i=1}^n, G)^{-1}]_{ji}
\end{align}
where $T\rightarrow 0 $ implies $t_i\rightarrow 0 ~~~\forall i\in [n]$. 

Since the equation~\eqref{eq:crude_new_res_dis} holds for $t_i \rightarrow 0$ for all $i\in [n]$, computing the value of limit for one trajectory should suffice for finding the value of the limit. 
Thereby, we provide an alternate equation for obtaining the resistance distance $r(i, j)$ by 
\begin{align}
    r(i, j) = [K(\pi_1 = i, D)]_{jj}
\end{align}
Note that $[K(\pi_1 = i, D)]_{ii} = [K(\pi_1 = i, D)]_{ij} = [K(\pi_1 = i, D)]_{ji} = 0$ from Lemma~\ref{lem:V_indep_T}). Thus we can say from Definition~\ref{def:d_better}, 
\begin{align}
    f(i, D) = \frac{1}{\mathfrak{I}(j, D)}\nonumber
\end{align}

Hence proved.

\end{proof}

\begin{lemma}\label{lem:optima_simplex_proof}
Let $D$ be a given graph with $n$ nodes. For every node $j\in D$, let $\{t_i^*(j)\}_{i\in D_j}$ denote the following:
\begin{align}
    \{t_i^*(j)\}_{i\in D_j} \in \underset{\sum_{i\in D_j, i\neq j} t_i = T}{\arg\max} \left\{[V_j(  \{t_i\}_{i\in D_j}, D)]^{-1}_{jj}- \frac{1}{T}\right\}
\end{align} 
Then $\exists \{t^*_i(j)\}_{i\in D_j}$, $l\in D_j$ such that $t_l^*(j) = T$ and $t_k^*(j) = 0~~\forall k\neq l$.
\end{lemma}
\begin{proof}
To simplify our proof, let graph $D$ be connected. The proof for the case of disconnected components is an extension of the connected graph case, by analysing each individual connected component together. 

If graph $D$ is connected then $D_i = D$. For the rest of the proof we sometimes denote $V(\pmb{\pi}_T, D)$ as $V(\{t_i\}_{i=1}^n, D)$ to make it more context relevant.

Let $g: \mathbb{R}^n \rightarrow \mathbb{R}^{n\times n}$ be a partial function of $V(\pmb{\pi}_T, D)$ as follows:
\begin{align}
    g(\{t_i\}_{i=1}^n) = V(\{t_i\}_{i=1}^n, D)
\end{align}
For all $i\in [n]$, let $t_i = \alpha_i T$ such that $\sum_{i=1}^n \alpha_i = 1$. Then we can say that,
\begin{align}
    g(\{t_i\}_{i=1}^n) &\ = g(\{\alpha_i T\}_{i=1}^n)\nonumber\\
    &\ = \sum_{i=1}^n \alpha_ig(\{0, 0, \dots t_i =  T, \dots 0\})
\end{align}
Using convexity of matrix invertibility~\cite{NORDSTROM20111489} $V(\pmb{\pi}_T, G)^{-1}$ satisfies,
\begin{align}
    g(\{t_i\}_{i=1}^n)^{-1} \preceq \sum_{i=1}^n \alpha_ig(\{0, 0, \dots t_i =  T, \dots 0\})^{-1}
\end{align}
Hence $g(\cdot)^{-1}$ is a convex function. Since we have the restriction as $\sum_{i=1, i\neq j}^n t_i = T$. We can say that,
\begin{align}
    &\ \underset{ \sum_{i\in D_j, i\neq j} t_i = T}{\arg\max} \left\{[V( \{t_i\}_{i=1}^n, D)]^{-1}_{jj} - \frac{1}{\sum_{i=1}^n t_i}\right\} \nonumber\\ &\ = \underset{ \sum_{i\in D_j, i\neq j} t_i = T}{\arg\max} [V( \{t_i\}_{i=1}^n, D)]^{-1}_{jj} \nonumber\\
    &\ = \underset{ \sum_{i\in D_j, i\neq j} t_i = T}{\arg\max} \langle\mathbf{e}_j, [V(\{t_i\}_{i=1}^n, D)]^{-1}\mathbf{e}_j\rangle\nonumber\\
    &\ = \underset{ \sum_{i\in D_j, i\neq j} t_i = T}{\arg\max} \langle\mathbf{e}_j, g(\{t_i\}_{i=1}^n)^{-1}\mathbf{e}_j\rangle
\end{align}
Since $g(\cdot)^{-1}$ is convex, for a convex function the maximization over a simplex happens at one of the vertices. Hence the max happens when $t_i=T$ and $t_k =0 ~~~\forall k\neq i$. 

Hence proved. 
\end{proof}

\begin{lemma}\label{lem:V_indep_T}
Let $G$ be a given connected graph of $n$ nodes and $t_i$ be the number of samples of each arm $i$. 
Then $\forall \pmb{\pi}_T \in \mathcal{U}(T)$,
\begin{align}\label{eq:inv_claim}
    V(\pmb{\pi}_T, G)^{-1} = \frac{1}{T}\mathbbm{1}\mathbbm{1}^T + K(\pi_1, G)
\end{align}
where, $\mathbbm{1} \in \mathbb{R}^{n}$ is a vector or all ones and $K(\pi_1, G) \in \mathbb{R}^{n\times n}$ is the matrix defined in Definition~\ref{def:d_better}.
\end{lemma}
\begin{proof}



Let $I$ be an identity matrix of dimension $n\times n$. We prove the result by showing that, $\forall \pmb{\pi}_T \in \mathcal{U}(T)$,  $V(\pmb{\pi}_T, G)^{-1}V(\pmb{\pi}_T, G) = I$, 
\begin{align}\label{eq:inv_T_arg}
    &\ V(\pmb{\pi}_T, G)^{-1}V(\pmb{\pi}_T, G) \nonumber\\ &\ = \left(\frac{1}{T}\mathbbm{1}\mathbbm{1}^T + K(\pi_1, G)\right)\left(\sum_{t=1}^T\mathbf{e}_{\pi_t}\mathbf{e}_{\pi_t}^T +\rho L_G\right) \nonumber\\
    &\ = \left(\frac{1}{T}\mathbbm{1}\mathbbm{1}^T + K(\pi_1, G)\right)\left(T\mathbf{e}_{\pi_1}\mathbf{e}_{\pi_1}^T +\rho L_G\right) \nonumber\\
    &\ = \mathbbm{1}\mathbf{e}_{\pi_1}^T + TK(\pi_1, G)\mathbf{e}_{\pi_1}\mathbf{e}_{\pi_1}^T +\rho K(\pi_1, G)L_G
\end{align}
From Definition~\ref{def:d_better}, $K(\pi_1, G)\mathbf{e}_{\pi_1}\mathbf{e}_{\pi_1}^T = 0$ and $\mathbbm{1}\mathbf{e}_{\pi_1}^T +\rho K(\pi_1, G)L_G = I$ implying that $V(\pmb{\pi}_T, G)^{-1}V(\pmb{\pi}_T, G) = I$. 


Hence proved. 
\end{proof}

\begin{lemma}\label{lem:min_v_t}
Let $G$ be any connected graph and $\pmb{\pi}_T \in \mathcal{U}(T, G)$. Then,
\begin{align}\label{eq:inv_claim_proof}
    \underset{j\in [n]}{\min}\{[V(\pmb{\pi}_T, G)^{-1}]_{jj}\} = \frac{1}{T}
\end{align}
\end{lemma}
\begin{proof}
From Definition~\ref{def:d_better}, $K(\pi_1, G)$ satisfies $$ K(\pi_1, G)L_G = \frac{1}{\rho}\left(I- \mathbbm{1}\mathbf{e}_{\pi_1}^T\right)$$

Observe that $\mathbbm{1}\mathbf{e}_i^T$ is  a rank 1 matrix with eigenvalue $1$ and eigenvector $\mathbf{e}_i$ and Identity matrix $I$ is of rank $n$ with all eigenvalues 1 and eigenvectors $\{\mathbf{e}_i\}_{i=1}^n$. Hence $\left(I- \mathbbm{1}\mathbf{e}_{\pi_1}^T\right)$ is a rank $n-1$ matrix with rest nonzero eigenvalues as $1$. Since the graph $G$ is connected, $\lambda_1(L_G) = 0$ and $\lambda_2(L_G) >0$. The eigenvector corresponding to $\lambda_1(L_G)$ is $\mathbb{1}$, the all $1$ vector. 

Given $\rho > 0$, we can conclude,
\begin{align}\label{eq:req_rank_later}
    K(\pi_1, G)L_G \succeq 0~~~\text{ s.t. rank}(K(\pi_1, G)L_G) = n-1
\end{align}
Hence, in order to satisfy eq.~\eqref{eq:req_rank_later}, $K(\pi_1, G) \succeq 0$ and $\text{rank}(K(\pi_1, G)) \geq n-1$. By lower bounds on Rayleigh quotient we can conclude,
\begin{align}
    \langle\mathbf{e}_j, K(\pi_1, G)\mathbf{e}_j\rangle = [K(\pi_1, G)]_{jj} \geq 0~~~\forall j\in [n]
\end{align}
From Lemma~\ref{lem:V_indep_T}, $ [K(\pi_1, G)]_{jj} = [V(\pmb{\pi}_T, G)^{-1}]_{jj} - \frac{1}{T}$ implying that $[V(\pmb{\pi}_T, G)^{-1}]_{jj} \geq \frac{1}{T}$. From Definition~\ref{def:d_better} it can be seen that $[K(\pi_1, G)]_{\pi_1\pi_1} = 0$ and hence $[V(\pmb{\pi}_T, G)^{-1}]_{\pi_1\pi_1} = \frac{1}{T}$ which concludes the proof.
\end{proof}

\begin{lemma}\label{lem:v_t_upper_bound}
Given a connected graph $G$, the following bound holds for all the diagonal entries of $[V(\pmb{\pi}_T, G)^{-1}]_{ii}$ for $i\in [n]$:
\begin{align}
    [V(\pmb{\pi}_T, G)^{-1}]_{ii}  \leq \mathbbm{1}\left(t_i = 0\right) \left(\frac{1}{\rho\mathfrak{I}(i, \mathcal{G})} + \frac{1}{T}\right) + \mathbbm{1}\left(t_i > 0\right)\max\left\{ \frac{1}{t_i + \frac{\rho\mathfrak{I}(i, G)}{2}}, \frac{1}{t_i + \frac{T}{2}} \right\}
\end{align}
\end{lemma}
\begin{proof}
From Definition~\ref{def:d_better} of $\mathfrak{I}(\cdot, \mathcal{G})$ and Lemma~\ref{lem:d_diff_forms},
Breaking the lemma statement into cases:
\begin{itemize}
    \item \textbf{Unsampled Arms :} From Lemma~\ref{lem:d_diff_forms}
\begin{align}
    \frac{1}{\mathfrak{I}(j, G)} = \underset{\sum_{i\in G_j, i\neq j} t_i = T}{\max} \left\{[V_j(  \{t_i\}_{i\in G_j}, G)^{-1}]_{jj} - \frac{1}{T}\right\}~~~\forall j \in [n]
\end{align}
Thus for any unsampled arm $j$, 
\begin{align}
    [V(\pmb{\pi}_T, G)]^{-1}_{jj} \leq \left(\frac{1}{\mathfrak{I}(j, G)} + \frac{1}{T}\right)
\end{align}
\item \textbf{Sampled Arms :} Since the matrix $V(\pmb{\pi}_T, G)$ depends only on the final sampling distribution $\{t_i\}_{i=1}^n$ rather than the sampling path $\pmb{\pi}_T$. Consider a sampling path such that $\pi_t \neq j$ for $t\leq T-t_j$ and $\pi_t = j$ for $T-t_j \leq t\leq T$. 

Assuming such a sampling path $\pmb{\pi}_T$, after $\pmb{\pi}_{T-t_j}$ samples,
\begin{align}
    [V(\pmb{\pi}_{T-t_j}, G)^{-1}]_{jj} \leq \frac{1}{T} + \frac{1}{\mathfrak{I}(j, G)}
\end{align}
Then by the Sherman-Morrison rank 1 update identity~\citep{Hager1989Inverse},
\begin{align}
    \frac{1}{[V(\pmb{\pi}_{T}, G)^{-1}]_{jj}} &\ = \frac{1}{[V(\pmb{\pi}_{T-t_j}, G)^{-1}]_{jj}} + t_j\nonumber\\
    [V(\pmb{\pi}_{T}, G)^{-1}]_{jj} &\ = \frac{1}{t_j + \frac{1}{[V(\pmb{\pi}_{T-t_j}, G)^{-1}]_{jj}}}\nonumber\\
    &\ \leq \frac{1}{t_j + \frac{1}{\left(\frac{1}{\mathfrak{I}(j, G)} + \frac{1}{T-t_j}\right)}}\nonumber
\end{align}
Hence we have the bound on ${[V(\pmb{\pi}_{T}, G)^{-1}]_{jj}}$ as follows:
\begin{align}
    [V(\pmb{\pi}_T, G)^{-1}]_{jj} \leq \max\left\{\frac{1}{t_j + \frac{\mathfrak{I}(j, G)}{2}}, \frac{1}{t_j + \frac{T-t_j}{2}}\right\}\nonumber\\
\end{align}
\end{itemize}
Hence proved. 
\end{proof}

\begin{lemma}\label{lem:completely_connected_d}
Let $D$ be a graph with $n$ nodes and $k$ disconnected components. If each of the connected components $\{\mathcal{C}_i(D)\}_{i=1}^k$ is a complete graph then $\forall~j\in [n]$,
\begin{align}
    \mathfrak{I}(j, D) = \frac{|\mathcal{C}(j, D)|}{2}
\end{align}
\end{lemma}
\begin{proof}
Let $D$ be a complete graph (k = 1), $\pmb{\pi}_T \in \mathcal{U}(T)$ and $\rho = 1$. Then,
\begin{align}\label{eq:inv_claim_proof_complete}
    V(\pmb{\pi}_T, G)^{-1} = \frac{1}{T}\mathbbm{1}\mathbbm{1}^T + K
\end{align}
where $\mathbbm{1} \in \mathbb{R}^{n}$ is a vector or all ones and $K \in \mathbb{R}^{n\times n}$ is a matrix given by,
\begin{align}
    &\ K_{\pi_1\pi_1} = 0,~~K_{jj} = \frac{2}{n}~~\forall j\in [n]/\{\pi_1\}\nonumber\\
    &\ K_{k\pi_1} = 0,~~K_{\pi_1j} = 0,~~K_{jk} = \frac{1}{n}~~\forall j, k\in [n]/\{\pi_1\},~ j\neq k\nonumber
\end{align}
The form of $V(\pmb{\pi}_T, G)^{-1}$ in eq.\eqref{eq:inv_claim_proof_complete} can be verified by $V(\pmb{\pi}_T, G)^{-1}V(\pmb{\pi}_T, G) = I$.

The final statement of the lemma can be obtained by considering this analysis to just the nodes within a connected component of a diconnected graph $G$ and Lemma~\ref{lem:d_diff_forms}.
\end{proof}

\begin{lemma}\label{lem:line_d}
Let $D$ be a graph with $n$ nodes and $k$ disconnected components. If each of the connected components $\{\mathcal{C}_i(D)\}_{i=1}^k$ is a line graph then $\forall~j\in [n]$,
\begin{align}
    \mathfrak{I}(j, D) > \frac{1}{|\mathcal{C}(j, D)|}
\end{align}
\end{lemma}
\begin{proof}
Let $D$ be a complete graph (k = 1), $\pmb{\pi}_T \in \mathcal{U}(T)$ and $\rho=1$. Then,
\begin{align}\label{eq:inv_claim_proof_line}
    V(\pmb{\pi}_T, G)^{-1} = \frac{1}{T}\mathbbm{1}\mathbbm{1}^T + K
\end{align}
where $\mathbbm{1} \in \mathbb{R}^{n}$ is a vector or all ones and $K \in \mathbb{R}^{n\times n}$ is a matrix given by,
\begin{align}
    &\ K_{\pi_1\pi_1} = 0,~~K_{jj} = d(\pi_1, j)~~\forall j\in [n]/\{\pi_1\}, \nonumber\\
    &\ K_{k\pi_1} = 0,~~K_{\pi_1j} = 0,\nonumber\\ &\ K_{jk} = \min\{d(\pi_1, j), d(\pi_1, k)\}~~\forall j, k\in [n]/\{\pi_1\},~ j\neq k\nonumber
\end{align}
The form of $V(\pmb{\pi}_T, G)^{-1}$ in eq.\eqref{eq:inv_claim_proof_line} can be verified by $V(\pmb{\pi}_T, G)^{-1}V(\pmb{\pi}_T, G) = I$.

The final statement of the lemma can be obtained by considering this analysis to just the nodes within a connected component of a diconnected graph $G$ and Lemma~\ref{lem:d_diff_forms}.
\end{proof}

\begin{lemma}\label{lem:d_change_subgraph}
    Let $A = ([n], E)$ be any graph and let $e\in E$ be an edge of graph $A$. Let $B = ([n], E-\{e\})$ be a subgraph of $A$ with one edge removed. Then the following holds for all non-isolated nodes $i$ in $B$:
    \begin{itemize}
        \item If $|\mathcal{C}(A)| = |\mathcal{C}(B)|$,
            \begin{align}
                \mathfrak{I}(i, A) \geq \mathfrak{I}(i, B)\nonumber
            \end{align}
        \item If $|\mathcal{C}(A)| < |\mathcal{C}(B)|$, 
            \begin{align}
                \mathfrak{I}(i, A) \leq \mathfrak{I}(i, B)\nonumber
            \end{align}
    \end{itemize}
\end{lemma}
\begin{proof}
From Lemma~\ref{lem:d_diff_forms}, for any graph $D$, $\mathfrak{I}(\cdot, \cdot)$ satisfies,
\begin{align}
    \frac{1}{\mathfrak{I}(j, D)} &\ = \underset{k\in D_j,\sum_{i\in D_j} t_i = T}{\max} \left\{[V_j(  \{t_i\}_{i\in D_j}, D)^{-1}]_{jj} \right. \nonumber\\ &\ \left. - [V_j(  \{t_i\}_{i\in D_j}, D)^{-1}]_{kk}\right\}~~~\forall j \in [n]
\end{align}
\textbf{Case I :} $|\mathcal{C}(A)| = |\mathcal{C}(B)|$ 

The edge set of $B$ is smaller than edge set of $A$. Hence, from Lemma~ $\mathfrak{I}(i, A) \geq \mathfrak{I}(i, B)$

\textbf{Case II :} $\mathcal{C}(A) < \mathcal{C}(B)$
In this case, $|B_i| \leq |A_i|$. Hence the $\max$ is over a smaller set of options, we can conclude that $\mathfrak{I}(i, A) \leq \mathfrak{I}(i, B)$. 
Hence proved. 
\end{proof}

Given a graph $D$, we define a class of sampling policies $\mathcal{U}(T, D)$ as follows, 
\begin{definition}
Let $\mathcal{U}(T, D)$ denote the set of sampling policies,
\begin{align}
    \mathcal{U}(T, D) = \{\pmb{\pi}_T |~\exists l\in D~\text{ s.t. } \pi_t = l ~~\forall t\leq T\}\nonumber
\end{align}
\end{definition}
\begin{lemma}
Let $G$ be the given graph and sampling policy $\pmb{\pi}_T$ has been played for $T$ time steps, then $V_T$ satisfies the following structure,
\begin{align}
    V(\pmb{\pi}_T, D) = \text{diag}([V_1, V_2, \dots, V_{k(G)}])
\end{align}
where $V_i$ depends on the connected component $C_i\in\mathcal{C}_D$ of the graph and the number of samples of the arms within the connected component $\{t_j\}_{j\in C_i}$. 
\end{lemma}
\begin{proof}
Rewriting the definition of $V(\pmb{\pi}_T, D)$,
\begin{align}
    V(\pmb{\pi}_T, D) &\ \triangleq \sum_{t=1}^T\mathbf{e}_{\pi_t}\mathbf{e}_{\pi_t}^\top +\rho L_D\nonumber\\
    &\ = N(\{t_i\}_{i=1}^n) + L_D
\end{align}
Both component matrices $N(\{t_i\}_{i=1}^n)$ (diagonal matrix) and $L_D$ (Laplacian matrix of a graph) adhere to a block diagonal structure and hence $V(\pmb{\pi}_T, D)$ matrix also adheres to a block diagonal structure analogous to $L_D$. The block diagonal structure in $L_D$ is dictated by connected components of graph $D$. 
\end{proof}

The following lemma establishes the invertibility of $V(\pmb{\pi}_T, G)$ for a connected graph and $T >1$ : 
\begin{lemma}\label{lem:connected_invert}
For a connected graph $G$, $V(\pmb{\pi}_1, G)$ is invertible, but $V(\pmb{\pi}_0, G)$ is not invertible. 
\end{lemma}
\begin{proof}
Since the graph $G$ is connected, $\lambda_1(L_G) = 0$ and $\lambda_2(L_G) >0$. The eigenvector corresponding to $\lambda_1(L_G)$ is $\mathbb{1}$, the all $1$ vector. At time $T = 0$, $V(\pmb{\pi}_T, G) = L_G$ and hence $V(\pmb{\pi}_T, G)$ is positive semi-definite matrix with one zero eigenvalues. 

Let arm $i$ be pulled at $T=1$, i.e. $\pi_1 = i$, then the corresponding counting matrix is a positive semi definite matrix of rank one with the eigen value $\lambda_n(N) = 1$ for the eigenvector $e_i$. 

Observe that $\mathbf{e}_i^T\mathbbm{1} >0$. Also, $N_T$ and $L_G$ are positive semi-definite matrices with ranks $1$ and $n-1$ respectively. The subspace without information (corresponding to the direction of zero eigenvalue) for matrix $L_G$ is now provided by $N(\pmb{\pi}_1)$ and hence $\lambda_{\min}(V(\pmb{\pi}_1, G)) >0$ making it invertible. 
\end{proof}

\begin{lemma}\label{lem:subgraph_worse_performance}
Let $G = ([n], E_G, A), H=([n], E_H, A)$ are two graphs with $n$ nodes such that $E_G \supseteq E_H$. Then, assuming invertibility of $[V(G, T)^{-1}]$ and $ [V(H, T)^{-1}]$, 
\begin{align}
    [t_{\text{eff}, i}]_G \geq [t_{\text{eff}, i}]_H~~~\forall i\in [n], T > k(G)
\end{align}
where $\forall i\in [n]$, $[t_{\text{eff}, i}]_G, [t_{\text{eff}, i}]_H$ indicates the effective samples with graph $G$ and $H$ respectively.  
\end{lemma}
\begin{proof}
Given graphs $G = ([n], E_G),H = ([n], E_H)$ satisfy $E_G \supseteq E_H$.

The quadratic form of Laplacian for the graph $G, H$ is given by,
\begin{align}
    \mathbf{x}L_G\mathbf{x} = \sum_{(i,j) \in E_G}(x_i - x_j)^2\nonumber\\
    \mathbf{x}L_H\mathbf{x} = \sum_{(i,j) \in E_H}(x_i - x_j)^2\nonumber
\end{align}
Since $E_G \supseteq E_H$, 
\begin{align}
    &\ \mathbf{x}L_G\mathbf{x} \geq \mathbf{x}L_H\mathbf{x}~~~\forall~\mathbf{x}\in \mathbb{R}^n\nonumber\\
    \Rightarrow &\ L_G \succeq L_H\nonumber
\end{align} 

Further, provided a sampling policy $\pmb{\pi}_T$, we can say that,
\begin{align}
    V(\pmb{\pi}_T, G) \succeq V(\pmb{\pi}_T, H)\nonumber
\end{align}
For the number of samples $T$ sufficient to ensure invertibility of $V(\pmb{\pi}_T, H)$, we have
\begin{align}
    &\ V(\pmb{\pi}_T, G)^{-1} \preceq V(\pmb{\pi}_T, H)^{-1}\nonumber\\
    &\ \mathbf{x}^TV(\pmb{\pi}_T, G)^{-1}\mathbf{x} \leq \mathbf{x}^TV(\pmb{\pi}_T H)^{-1}\mathbf{x}~~~~\forall \mathbf{x}\in \mathbb{R}^n\nonumber\\
    &\ [V(\pmb{\pi}_T, G)^{-1}]_{ii} \leq [V(\pmb{\pi}_T, H)^{-1}]_{ii}~~~~(\text{taking }\mathbf{x}=\mathbf{e}_i)\nonumber\\
    &\ \frac{1}{[V(\pmb{\pi}_T, G)^{-1}]_{ii}} \geq \frac{1}{[V(\pmb{\pi}_T, H)^{-1}]_{ii}}\nonumber
\end{align}
Hence from the definition of effective samples~\ref{def:effective_samples}, it is clear that for any $i\in [n]$,
\begin{align}
[t_{\text{eff}, i}]_G \geq [t_{\text{eff}, i}]_H
\end{align}

Hence proved.
\end{proof}

\begin{lemma}\label{lem:infinite_sum_bounded}
Let effective samples $t_{\text{eff}, i}$ be as is defined in Definition~\ref{def:effective_samples} and let $\pmb{\pi}_T$ denote a cyclic sampling policy for $T> k(G)$, then the infinite sum $\sum_{T=k(G)+1}^\infty t_{\text{eff}, i}^{-2}$ is bounded. In fact, 
\begin{align}
    \sum_{T=k(G)+1}^\infty t_{\text{eff}, i}^{-2} < n\left(\frac{2(n-1)}{\rho}\right)^{2} + \frac{n\pi^2}{6}
\end{align}
\end{lemma}
\begin{proof}
We first prove the lemma statement for connected graph $G$ and then go towards a more general graph $G$. From Lemma~\ref{thm:effective_samples_proof}, 
\begin{align}
    t_{\text{eff}, i} \geq t_i + \min\{\rho\mathfrak{I}(i, G) , T-t_i\}\nonumber
\end{align} 
if $T-t_i \leq \rho\mathfrak{I}(i, G)$, then $t_{\text{eff}, i} \geq \frac{T+t_i}{2}\geq \frac{T}{2}$. For the reverse case of $T-t_i \geq \rho\mathfrak{I}(i, G)$, $t_{\text{eff}, i} \geq t_i+\frac{\rho\mathfrak{I}(i, G)}{2}\geq t_i + \frac{\rho}{2(n-1)}$ (since $\mathfrak{I}(i, G) \geq \frac{1}{n-1}$ by Lemma~\ref{lem:v_t_upper_bound}). 

Since $\pmb{\pi}_T$ is a cyclic sampling policy, hence $t_i$ increases by 1 at-least once every $n$ samples. Thus, we can upperbound the infinite sum as,
\begin{align}
    \sum_{T=1}^\infty \frac{1}{t_{\text{eff}, i}^2} &\ \leq \sum_{T=1}^\infty \frac{1}{\left(t_i + \frac{\rho}{2(n-1)}\right)^2}\nonumber\\
    &\ \leq n\left(\frac{2(n-1)}{\rho}\right)^{2} + n\sum_{t_i=1}^\infty\frac{1}{t_i^2}\nonumber\\
    &\ < n\left(\frac{2(n-1)}{\rho}\right)^{2} + \frac{n\pi^2}{6}
\end{align}
Hence proved.
\end{proof}

\section{Code Availability}\label{app:code}
The full code used for conducting experiments can be found at the following \color{blue}\href{https://github.com/parththaker/Bandits-GRUB}{Github repository}\color{black}.

\end{document}